%% file: main.tex
\newcommand{\shortversion}[1]{}
\newcommand{\longversion}[1]{#1}

\documentclass[letterpaper, 10 pt, conference]{ieeeconf}  

\IEEEoverridecommandlockouts                              

\usepackage{amsmath} 
\usepackage{amsfonts}
\usepackage{xcolor}
\usepackage[linesnumbered,ruled,noend, noline]{algorithm2e}
\usepackage{romannum}
\usepackage{graphicx}
\usepackage{caption}
\usepackage{subcaption}
\usepackage[export]{adjustbox}
\usepackage[switch]{lineno}
\usepackage{multirow}
\usepackage{booktabs}
\usepackage{algorithmicx}
\usepackage{todonotes}
\usepackage{siunitx}        
\usepackage{rotating}       
\usepackage{soul}
\usepackage{tikz}           
\usepackage{multicol}       
\usepackage{url}            

\usetikzlibrary{fit,calc}   

\definecolor{mygray}{gray}{0.9}
\sethlcolor{mygray}

\graphicspath{{figures/}}

\SetCommentSty{mycommfont}

\SetKw{Break}{break}
\SetKw{Continue}{continue}
\SetKw{Exit}{exit}
\SetKw{Return}{return}

\SetKwFunction{FnRobot}{Robot}
\SetKwFunction{FRcvPath}{receive\_path}
\SetKwFunction{FRcvLV}{receive\_localview}
\SetKwFunction{FnUpdateGlobalview}{update\_globalview}
\SetKwFunction{FGetEqLenPaths}{get\_equal\_length\_paths}
\SetKwFunction{FnSndPathsToActRobs}{send\_paths\_to\_active\_robots}
\SetKwFunction{FSndPath}{send\_path}
\SetKwFunction{FnOnDemCPP}{OnDemCPP}
\SetKwFunction{FnOnDemCPPHor}{OnDemCPP\_Hor}
\SetKwFunction{FnCPPForPar}{CPPForPar}
\SetKwFunction{FnCOPForPar}{COPForPar}
\SetKwFunction{FnCFPForPar}{CFPForPar}

\SetKwFunction{FnGAMRCPP}{GAMRCPP}
\SetKwFunction{FnGAMRCPPMax}{GAMRCPP\textsubscript{MAX}}
\SetKwFunction{FnGAMRCPPHorizon}{GAMRCPP\_Horizon}

\SetKwProg{Fn}{Function}{:}{}
\SetKwProg{Srv}{Service}{:}{}

\newtheorem{example}{Example}
\newtheorem{problem}{Problem}
\newtheorem{theorem}{Theorem}
\newtheorem{lemma}{Lemma}

\newcommand{\ParFor}[2]{\textbf{parallel for }#1 \textbf{do}\\\indent#2}

\shortversion{
\renewcommand{\baselinestretch}{0.98}
}

\newcommand{\RomanNumeralCaps}[1]{\MakeUppercase{\romannumeral #1}}

\title{\LARGE \bf
Online On-Demand Multi-Robot Coverage Path Planning
}

\author{Ratijit Mitra$^{1}$ and Indranil Saha$^{2}$
\thanks{$^{1}$Ratijit Mitra and $^{2}$Indranil Saha are with the Department of Computer Science and Engineering, Indian Institute of Technology Kanpur, Uttar Pradesh - 208016, India
        {\tt\small \{ratijit,isaha\}@iitk.ac.in}}%
}

\begin{document}

\maketitle
\thispagestyle{empty}
\pagestyle{empty}

\begin{abstract}
\input{1_abstract}
\end{abstract}

\input{2_introduction}
\input{4_problem}
\input{5_solution}
\input{6_theoretical_analysis}
\input{7_evaluation}
\input{8_conclusion}

\bibliographystyle{IEEEtran}
\bibliography{9_references_short}
\end{document}

%% file: 1_abstract.tex
We present an online centralized path planning algorithm to cover a large, complex, unknown workspace with multiple homogeneous mobile robots. 
Our algorithm is horizon-based, synchronous, and on-demand. 
The recently proposed horizon-based synchronous algorithms compute all the robots' paths in each horizon, significantly increasing the computation burden in large workspaces with many robots. 
As a remedy, we propose an algorithm that computes the paths for a subset of robots that have traversed previously computed paths entirely (thus on-demand) and reuses the remaining paths for the other robots. 
We formally prove that the algorithm guarantees complete coverage of the unknown workspace. 
Experimental results on several standard benchmark workspaces show that our algorithm scales to hundreds of robots in large complex workspaces and consistently beats a state-of-the-art online centralized multi-robot coverage path planning algorithm in terms of the time needed to achieve complete coverage. 
For its validation, we perform ROS+Gazebo simulations in five $2$D grid benchmark workspaces with $10$ Quadcopters and $10$ TurtleBots, respectively. 
Also, to demonstrate its practical feasibility, we conduct one indoor experiment with two real TurtleBot2 robots and one outdoor experiment with three real Quadcopters. 

%% file: 2_introduction.tex
\section{Introduction}
\label{sec-intro}

Coverage Path Planning (CPP) deals with finding conflict-free routes for a fleet of robots to make them completely visit the obstacle-free regions of a given workspace to accomplish some designated task. 
It has numerous applications in indoor environments, e.g.,  vacuum cleaning \cite{DBLP:conf/icra/BormannJHH18, DBLP:conf/icra/VandermeulenGK19}, industrial inspection \cite{DBLP:conf/iros/JingPGRLS17}, etc.,  as well as in outdoor environments, e.g.\longversion{, road sweeping \cite{DBLP:conf/icra/EngelsonsTH22}, lawn mowing \cite{DBLP:conf/iros/SchirmerBS17}}, precision farming~\cite{DBLP:journals/jfr/BarrientosCCMRSV11}\longversion{, \cite{DBLP:conf/atal/LuoNKS19}}, surveying~\cite{DBLP:conf/icra/KarapetyanMLLOR18}\longversion{, \cite{DBLP:conf/icra/AgarwalA20}}, search and rescue operations \longversion{\cite{DBLP:conf/iros/LewisEBRO17}, }\cite{cabreira2019survey}, \longversion{demining a battlefield~\cite{DBLP:conf/icra/DogruM19}, }etc. 
A CPP algorithm, often called a Coverage Planner (CP), is said to be \textit{complete} if it guarantees coverage of the entire obstacle-free region. 
Though a \textit{single} robot is enough to achieve complete coverage of a small environment (e.g., \cite{DBLP:conf/iros/BouzidBS17, DBLP:conf/iros/KleinerBKPM17, DBLP:conf/atal/SharmaDK19}\longversion{, \cite{DBLP:conf/icra/WeiI18, DBLP:conf/icra/CoombesCL19}}, \cite{DBLP:conf/iros/ChenTKV19, DBLP:conf/icra/DharmadhikariDS20}\longversion{, \cite{DBLP:conf/icra/ZhuYLSJM21}}), 
\textit{multiple} robots (e.g., \cite{DBLP:journals/ras/GalceranC13, DBLP:conf/icra/ModaresGMD17, DBLP:conf/iros/KarapetyanBMTR17}\longversion{, \cite{DBLP:conf/atal/SalarisRA20, DBLP:conf/icra/AgarwalA20}}, \cite{DBLP:conf/icra/VandermeulenGK19, DBLP:conf/iros/HardouinMMMM20, DBLP:conf/icra/TangSZ21, DBLP:conf/icra/CollinsGEDDC21}) facilitate complete coverage of a large environment more quickly. 
However, the design complexity of the CP grows significantly to exploit the benefit of having multiple robots. 

For many CPP applications, the workspace's obstacle map is unknown initially. 
So, the \textit{offline} CPs (e.g., \longversion{\cite{DBLP:conf/icra/HazonK05, DBLP:conf/iros/ZhengJKK05, DBLP:conf/atal/RivaA17}, }\cite{DBLP:conf/iros/KleinerBKPM17, DBLP:conf/icra/BormannJHH18}\longversion{, \cite{DBLP:conf/atal/SalarisRA20, DBLP:conf/icra/CaoZTC20, DBLP:conf/iros/RamachandranZPS20, DBLP:conf/icra/SongYQSLL0022}}) that require the obstacle map beforehand are not applicable here.
Instead, we need an \textit{online} CP (e.g., \cite{DBLP:conf/agents/Yamauchi98}\longversion{, \cite{DBLP:journals/ijrr/HowardPS06}}, \cite{DBLP:conf/icra/BircherKAOS16}\longversion{, \cite{DBLP:conf/mesas/MannucciNP17, DBLP:conf/atal/JensenG18}}, \cite{DBLP:conf/iccps/DasS18, DBLP:conf/icra/OzdemirGKHG19, DBLP:conf/iros/HardouinMMMM20}\longversion{, \cite{DBLP:conf/iros/DattaA21}}) that runs through multiple rounds to cover the entire workspace gradually. 
In each round, the robots explore some unexplored regions using attached sensors, and the CP subsequently finds their subpaths to cover the explored obstacle-free regions not covered so far, a.k.a. \textit{goals}. 

Based on where the CP runs, we can classify it as either \textit{centralized} or \textit{distributed}. 
A centralized CP (e.g., \cite{DBLP:conf/icra/GabrielyR01, DBLP:conf/agents/Yamauchi98, DBLP:conf/icra/BircherKAOS16}, \longversion{\cite{DBLP:conf/atal/RivaA17}, }\cite{DBLP:conf/iros/KleinerBKPM17, DBLP:conf/iccps/DasS18, DBLP:conf/atal/SharmaDK19}\longversion{, \cite{DBLP:conf/atal/SalarisRA20, DBLP:conf/icra/LiRS21}}, \cite{DBLP:conf/iros/MitraS22}) runs at a server and is responsible for finding all the paths alone. 
In contrast, a distributed CP (e.g., \longversion{\cite{butler2001complete}, }\cite{DBLP:conf/icra/HazonMK06}\longversion{, \cite{DBLP:journals/amai/RekleitisNRC08, DBLP:conf/ijcai/JensenG13, DBLP:conf/dars/HungerfordDG14}}, \cite{DBLP:journals/apin/VietDCC15}\longversion{, \cite{DBLP:conf/icra/SiligardiPKMGBS19, DBLP:journals/arobots/SongG20, DBLP:conf/iros/HassanML20}}) runs at every robot as a local instance, and these local instances \textit{collaborate} among themselves to find individual paths. 
The distributed CPs are computationally faster as they deal with only the local state spaces. 
However, they find highly inefficient paths due to the lack of global knowledge about the state space. 
Despite having high computation time, the centralized CPs can provide a shorter coverage 
completion time as they can find highly efficient paths by exploiting the global state space. 
The recently proposed \textit{receding horizon}-based (e.g., \cite{DBLP:conf/icra/BircherKAOS16, DBLP:conf/iccps/DasS18}) online multi-robot centralized coverage planner \FnGAMRCPP~\cite{DBLP:conf/iros/MitraS22} demonstrates this capability, thereby outperforming the state-of-the-art online multi-robot distributed coverage planner $\mathtt{BoB}$~\cite{DBLP:journals/apin/VietDCC15}. 

A major bottleneck of \FnGAMRCPP is that it generates the paths for \textit{all} the robots \textit{synchronously} (i.e., at the same time) in each \textit{horizon} (like \cite{DBLP:conf/iccps/DasS18}), which prevents it from scaling for hundreds of robots in large workspaces. 
Furthermore, \FnGAMRCPP decides the \textit{horizon length} based on the \textit{minimum} path length and \textit{discards} the remaining paths for the robots with longer paths. 
Finding a better goal assignment for those robots in the next horizon is possible, but discarding the already generated paths leads to considerable computational wastage. 
In this paper, we propose an alternative approach where, like \FnGAMRCPP, the CP decides the horizon length based on the minimum path length but \textit{keeps} the remaining paths for the robots with longer paths for traversal in the subsequent horizons. 
Thus, in a horizon, our proposed CP has to synchronously generate the paths for only those robots for which no remaining path is available (called the \textit{participant robots}), hence \textit{on-demand}. 
However, this new on-demand approach brings the additional challenge of computing the new paths under the constraint of the remaining paths for some robots. 
Our CP tackles this challenge soundly. 
Though the proposed approach misses the opportunity to find more optimal paths for the \textit{non-participant robots} in a horizon, the computation load in each horizon decreases significantly, which in turn leads to a faster coverage completion of large workspaces with hundreds of robots, promising \textit{scalability}. 

We formally prove that the proposed CP can achieve \emph{complete coverage}. 
To evaluate its performance, we consider eight large $2$D grid-based benchmark workspaces of varying size and obstacle density, and two types of robots, TurtleBot\longversion{ \cite{key_turtlebot}}, which is a ground robot, and a Quadcopter, which is an aerial robot, for their coverage. 
We vary the number of robots from $128$ to $512$ and choose \textit{mission time} as the comparison metric, which is the time required to attain complete coverage. 
We compare our proposed CP with \FnGAMRCPP and show that it outperforms \FnGAMRCPP consistently in large workspaces involving hundreds of robots. We further demonstrate the practical feasibility of our algorithm through ROS+Gazebo simulations and real experiments. 

%% file: 4_problem.tex
\section{Problem}
\label{sec:problem}

\subsection{Preliminaries}
\label{subsec:preliminaries}

Let $\mathbb{R}$ and $\mathbb{N}$ denote the set of real numbers and natural numbers, respectively, and $\mathbb{N}_0$ denote the set $\mathbb{N} \cup \{0\}$. 
Also, for $m \in \mathbb{N}$, we write $[m]$ to denote the set $\{n \in \mathbb{N}\, |\, n \leq m\}$, and $[m]_0$ to denote the set $[m] \cup \{0\}$. 
The size of the countable set $\mathcal{S}$ is denoted by $|\mathcal{S}| \in \mathbb{N}_0$. 
Furthermore, we denote the set $\{0, 1\}$ of Boolean values by $\mathbb{B}$.

\subsubsection{Workspace}
\label{subsubsec:workspace}

We consider an unknown $2$D workspace $W$ represented as a grid of size $X \times Y$, where $X, Y \in \mathbb{N}$. 
Thus, $W$ is represented as a set of \textit{non-overlapping square-shaped grid cells} 
\mbox{$\{(x, y) \, | \, x \in [X] \wedge y \in [Y]\}$}, some of which are \textit{obstacle-free} (denoted by $W_{free}$) and \textit{traversable} by the robots, while the rest are \textit{static obstacle-occupied} (denoted by $W_{obs}$), which are not. 
Note that $W_{free}$ and $W_{obs}$ are not known initially. 
We assume that $W_{free}$ is \textit{strongly connected} and $W_{obs}$ is \textit{fully} occupied with obstacles.

\subsubsection{Robots and their States}
\label{subsubsec:robots}

We employ a team of \mbox{$R \in \mathbb{N}$} \textit{failure-free homogeneous mobile} robots, where each robot fits entirely within a cell. 
We denote the $i (\in [R])$-th robot by $r^i$ and assume that $r^i$ is \textit{location-aware}. 
Let the state of $r^i$ at the $j (\in \mathbb{N}_0)$-th \textit{discrete} time step be $s^i_j$, which is a tuple of its \textit{location} and possibly \textit{orientation} in the workspace. 
We define a function $\mathcal{L}$ that takes a state $s^i_j$ as input and returns the corresponding location as a tuple. 
Initially, the robots get deployed at different obstacle-free cells, comprising \textit{the set of start states} $S = \{s^i_0\ |\ i \in [R] \wedge \mathcal{L}(s^i_0) \in W_{free} \wedge \forall j \in [R] \setminus \{i\}.\ \mathcal{L}(s^j_0) \neq \mathcal{L}(s^i_0)\}$. 
We also assume that each $r^i$ is equipped with four \textit{rangefinders} on all four sides to detect obstacles in the four neighboring cells they are facing.

\subsubsection{Motions and Paths of the robots}
\label{subsubsec:motions}

The robots have a common \textit{set of motion primitives} $M$ to change their states in the next time step. 
It also contains a unique motion primitive $\mathtt{Halt (H)}$ to keep the state of a robot unchanged in the next step. 
Each motion primitive $\mu \in M$ is associated with some cost $\mathtt{cost}(\mu) \in \mathbb{R}$, e.g., distance traversed, energy consumed, etc. 
We assume that all the motion primitives take the same $\tau \in \mathbb{R}$ unit time for execution. 
Initially, the path $\pi^i$ for robot $r^i$ contains its start state $s^i_0$. 
So, the \textit{length} of $\pi^i$, denoted by $\Lambda \in \mathbb{N}_0$, is $0$. 
But, when a \textit{finite sequence} of motion primitives $(\mu_j \in M)_{j \in [\Lambda]}$ of length $\Lambda > 0$ gets applied to $s^i_0$, it results in generating the $\Lambda$-length path $\pi^i$. 
So, $\pi^i$ is a finite sequence $(s^i_j)_{j \in [\Lambda]_0}$ of length $\Lambda + 1$ of the states of $r^i$: 

\begin{center}
    $s^i_{j - 1} \xrightarrow{\mu_{j}} s^i_j$, $\forall {j \in [\Lambda]}$.
\end{center}

Thus, $\pi^i$ has the cost $\mathtt{cost}(\pi^i) = \sum_{j \in [\Lambda]} \mathtt{cost}(\mu_j)$. 
Note that we can make \textit{the set of paths} $\Pi = \{\pi^i | i \in [R]\}$ of all the robots \textit{equal-length} $\Lambda$ by suitably applying $\mathtt{H}$ at their ends. 

\input{4a_turtlebot}

\subsection{Problem Definition}
\label{subsec:problem_definition}

We now formally define the CPP problem below. 

\begin{problem}[Complete Coverage Path Planning]
Given an unknown workspace $W$, start states $S$ of $R$ robots, and their motion primitives $M$, find paths $\Pi$ of equal-length $\Lambda$ for the robots such that the following two conditions hold:

\textbf{Cond. \romannum{1}:} Each path $\pi^i$ must satisfy the following:
    \begin{enumerate}
        \item $\forall j \in [\Lambda]_0\ \mathcal{L}(s^i_j) \in W_{free}$,\ \ \ \ \ \ \ \ \ \ \ \ \ \ \ \ \ [\textit{Avoid obstacles}]
        \item $\forall j \in [\Lambda]_0\ \forall k \in [R] \setminus \{i\}\\
        \mathcal{L}(s^i_j) \neq \mathcal{L}(s^k_j)$,\ \ \ \ \ \ \ \ \ \ \ \ \ \ \ \ \ [\textit{Avert same cell collisions}]
        \item $\forall j \in [\Lambda]\ \forall k \in [R] \setminus \{i\}\\
        \lnot((\mathcal{L}(s^i_{j-1}) = \mathcal{L}(s^k_j)) \wedge (\mathcal{L}(s^i_j) = \mathcal{L}(s^k_{j-1})))$.\\
        \textcolor{white}{BLANK}\ \ \ \ \ \ \ \ \ \ \ \ \ \ \ \ \ \ \ \ \ \ \ \ \ \ \ \ \ [\textit{Avert head-on collisions}]
    \end{enumerate}

\textbf{Cond. \romannum{2}:} Each obstacle-free cell must get visited by at least one robot, i.e., $\bigcup\limits_{i \in [R]}\ \bigcup\limits_{j \in [\Lambda]_0} \{\mathcal{L}(s^i_j)\} = W_{free}$.
\end{problem}


%% file: 4a_turtlebot.tex
\begin{example}
\longversion{For illustration, we consider TurtleBot \cite{key_turtlebot2}. }
A TurtleBot~\shortversion{\cite{key_turtlebot2}} not only \textit{drives} forward to change its location but also \textit{rotates} around its axis to change its orientation. 
So, the state of a TurtleBot is $s^i_j = (x, y, \theta)$, where $(x, y) \in W$ is its location in the workspace and $\theta \in \{\mathtt{East (E), North (N), West (W), South (S)}\}$ is its orientation at that location. 
The set of motion primitives is given by $M = \{\mathtt{Halt (H)}, \mathtt{TurnRight (TR)}, \mathtt{TurnLeft (TL)}, \mathtt{MoveNext (MN)}\}$, where  $\mathtt{TR}$ turns the TurtleBot $90^{\circ}$ \textit{clockwise}, $\mathtt{TL}$ turns the TurtleBot $90^{\circ}$ \textit{counterclockwise}, and $\mathtt{MN}$ \textit{moves} the TurtleBot to the next cell pointed by its orientation $\theta$. 
\end{example}

%% file: 5_solution.tex
\section{On-Demand CPP Framework}
\label{sec:asyn_cpp}

This section presents the proposed centralized horizon-based online multi-robot on-demand CPP approach for complete coverage of an unknown workspace whose size and boundary are only known to the CP and the mobile robots. 

Due to the limited range of the fitted rangefinders, each robot gets a \textit{partial} view of the unknown workspace, called the \textit{local view}. 
Initially, all the robots share their initial local views with the CP by sending \textit{requests} for paths. 
The CP then fuses these local views to get the \textit{global view} of the workspace. 
Based on the global view, it attempts to generate collision-free paths for the robots, possibly of different lengths. 
The robots with \textit{non-zero-length} paths are said to be \textit{active}, while the rest are \textit{inactive}. 
Next, the CP determines the horizon length as the minimum path length of the active robots. 
Subsequently, it makes the paths of the active robots of length equal to that horizon length. 
If any active robot's path length exceeds the horizon length, the CP stores the remaining part for future horizons. 
Finally, the CP provides these equal-length paths to respective active robots by sending \textit{responses}. 
As the CP has failed to find paths for the inactive robots in the current horizon, it considers them again in the next horizon. 
However, the active robots follow their received paths synchronously and update their local views accordingly. 
Upon finishing the execution of their current paths, they share their updated local views with the CP again. 
In the next horizon, after updating the global view, the CP attempts to generate paths for the previous horizon's inactive robots and active robots who have completed following their last planned paths. 
In other words, the CP generates paths for only those robots with no remaining path, called the \textit{participants}.
Note that the robots with remaining paths computed in the previous horizon become the \textit{non-participants} (inherently active) in the current horizon. 
Moreover, the CP does not alter the non-participants' remaining paths while generating the participants' paths in the current horizon. 

\begin{figure}[t]
    \centering
    \includegraphics[scale=0.35]{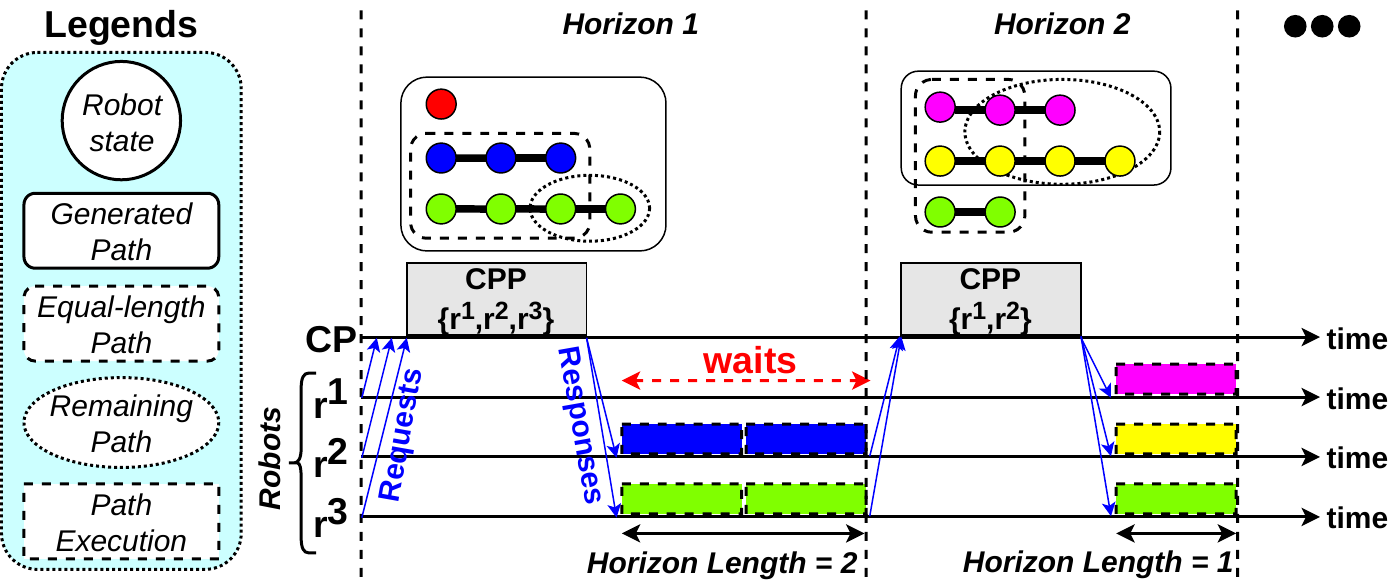}
    \caption{Overview of on-demand CPP}
    \label{fig:asyncpp}
\end{figure}

\begin{example}
\label{ex:update_gv}
In Figure~\ref{fig:asyncpp}, we show a schematic diagram of the on-demand CPP approach for three robots in an arbitrary workspace. 
In horizon $1$, the CP finds paths of lengths $0, 2$, and $3$ for the participants $r^1, r^2$, and $r^3$, respectively. 
Observe that only $r^2$ and $r^3$ are active. 
As the horizon length turns out to be $2$, the CP stores the remaining part of $r^3$'s path. 
Now, $r^2$ and $r^3$ follow their equal-length paths while $r^1$ waits. 
In horizon $2$, however, the CP finds paths only for $r^1$ and $r^2$ (notice the new colors) as non-participant $r^3$ already has its remaining path (notice the same color) found in horizon $1$. 
\end{example}

\input{5a_ds.tex}
\input{5b_overall.tex}
\input{5c_horizon.tex}
\shortversion
{
\input{5d_participants_short.tex}
}
\longversion
{
\input{5d_participants.tex}
}

%% file: 5a_ds.tex
\subsection{Data structures at CP}
\label{subsec:ds}

Before delving deep into the processes that run at each robot $r^i$ (Algorithm~\ref{algo:robot}) and the CP (Algorithm \ref{algo:ondemcpp}), we define the data structures that the CP uses for computation. 

Let $\Sigma_{rem}$ denote the set $\{\sigma^i_{rem} \, |\, i \in [R]\}$ of the remaining paths of all the robots for the next horizon, where $\sigma^i_{rem}$ denotes the remaining path for $r^i$. 
Initially, none of the robots has any remaining path, i.e., $\sigma^i_{rem}$ is $\mathsf{NULL}$ for each $r^i$. 
We denote the set of indices of the robots participating in the current horizon by $I_{par} \subseteq [R]$ and the number of requests received in the current horizon by $N_{req} \in [R]_0$. 
Also, $N_{act} \in [R]_0$ denotes the number of active robots in the current horizon. 
By $I^{inact}_{par}$, we denote the set of indices of the inactive robots of the previous horizon, and by $S^{inact}$, their start states. 
So, the CP must incorporate the robots $I^{inact}_{par}$ into the planning of the current horizon. 
Finally, $\widetilde{W_g}$ denotes the goals at which the remaining paths end. 
So, $\widetilde{W_g}$ remains reserved for the corresponding robots, and the CP cannot assign $\widetilde{W_g}$ to any other robot hereafter. 

%% file: 5b_overall.tex
\subsection{Overall On-Demand Coverage Path Planning}
\label{subsec:overall_approach}

In this section, we present the overall on-demand CPP approach in detail. First, we explain Algorithm \ref{algo:robot}, which runs at the robots' end. 
Each robot $r^i$ initializes its local view $W^i = \langle W^i_u, W^i_o, W^i_g, W^i_c \rangle$ (line 1), where $W^i_u, W^i_o, W^i_g,$ and $W^i_c$ are the set of unexplored, obstacle, goal, and covered cells, respectively, due to the partial visibility of the workspace. 
Subsequently, $r^i$ creates a request message $\mathtt{M_{req}[id, state, view]}$ 
(line 2) and sends that to the CP (line 4) to inform about its local view. 
Then, $r^i$ waits for a response message $\mathtt{M_{res}[path]}$ (line 5) from the CP. 
Upon receiving, $r^i$ extracts its path into $\sigma^i$ and starts to follow it (line 6). 
While following, $r^i$ explores newly visible cells and accordingly updates $W^i$. 
Finally, once it reaches the last state $s^i_{|\sigma^i|}$ of $\sigma^i$, it updates its $\mathtt{M_{req}}$ (line 7) to inform the CP about its updated $W^i$. 
The same communication cycle continues between $r^i$ and the CP in a \textbf{while} loop (lines 3-7). 
Here, we assume that the \textit{communication medium} is \textit{reliable} and each communication takes negligible time. 

\input{algorithms/robot}

\shortversion
{
\input{algorithms/ondemcpp_v2}
}
\longversion
{
\input{algorithms/ondemcpp}
}

We now describe Algorithm~\ref{algo:ondemcpp} in detail, which runs at the CP's end. 
After initializing the required data structures (lines~1-4), the CP starts a \textit{service} (lines~5-18) to receive requests from the robots in a mutually exclusive manner. 
Upon receiving a $\mathtt{M_{req}}$ from $r^i$ (line 6), the CP checks whether $r^i$ has any remaining path $\sigma^i_{rem}$ or not (line 7). 
If $r^i$ has, it does not participate in the planning of the current horizon. 
Otherwise (lines~7-9), $r^i$ participates, and so the CP adds $r^i$'s ID $\{i\}$ into $I_{par}$, 
and current state into $S$. 
In either case, the CP updates the global view of the workspace $W = \langle W_u, W_o, W_g, W_c \rangle$ \textit{incrementally} (line 10). 
\shortversion{In the longer version of this paper \cite{DBLP:conf/arxiv/MitraS23}, please refer to lines 19-32 of Algorithm \ref{algo:ondemcpp} for the definition of the function \FnUpdateGlobalview, the third paragraph of Section \ref{subsec:overall_approach} for its description, and see Example 3 for an illustration.}

\longversion
{
In \FnUpdateGlobalview (lines 19-32), first, the CP extracts the local view $W^i$ of $r^i$ (line 20). 
Then, it marks all the cells visited by $r^i$ as covered in the global view $W$ (lines 21-23). 
Next, the CP marks any unexplored cell in $W$ as a goal if $r^i$ classifies that cell as a goal (lines 24-26). 
Similarly, the CP marks any unexplored cell in $W$ as an obstacle if $r^i$ classifies that cell as an obstacle (lines 28-30). 
As $r^i$ visits some previously found goals, making them covered, the CP filters out those covered cells from $W_g$ (line 27). 
Also, as a robot cannot distinguish between an obstacle and another robot using the fitted 
rangefinders, the CP filters out those covered cells from $W_o$ that previously posed as obstacles (line 31). 
Finally, the CP marks the remaining cells in $W$ as unexplored (line 32). 

\begin{figure}
    \centering
    \includegraphics[scale=0.32]{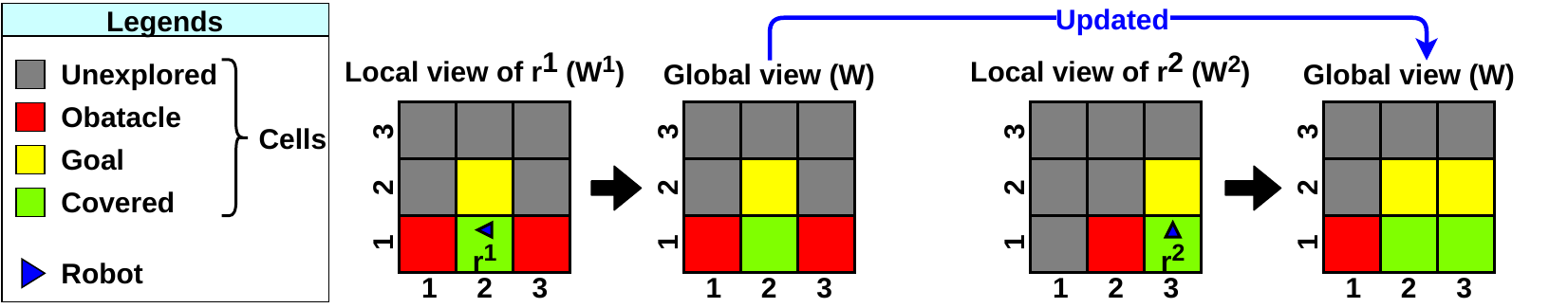}
    \caption{Incremental updation of the global view $W$}
    \label{fig:update_gv}
\end{figure}

\begin{example}
\label{ex:update_gv}
In Figure \ref{fig:update_gv}, we show an example, where two robots, namely $r^1$ and $r^2$ with initial states $(2, 1, \mathtt{W})$ and $(3, 1, \mathtt{N})$, respectively, get deployed in a $3 \times 3$ workspace. 
Notice that $r^2$ appears as an obstacle to $r^1$ in its local view $W^1$. 
Likewise, $r^1$ appears as an obstacle to $r^2$ in its local view $W^2$. 
Without loss of generality, we assume the CP receives requests from $r^1$ first and then from $r^2$ and incrementally updates the global view $W$ accordingly. 
\end{example}

\begin{figure*}[t]
    \centering
    \includegraphics[scale=0.275]{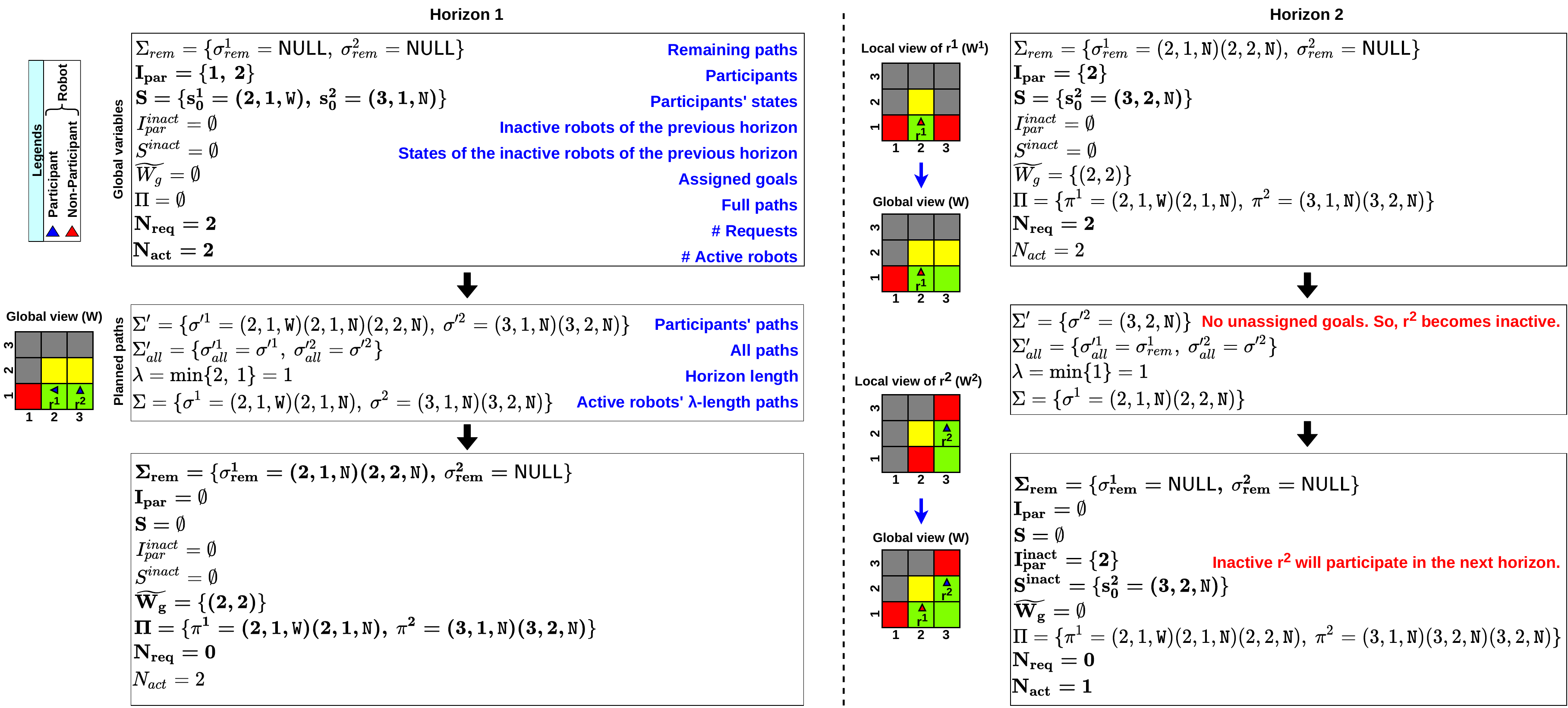}      
    \caption{Skip replanning the participants in the current horizon}
    \label{fig:skip_replan_par}
\end{figure*}
}

Once the CP receives all the requests from the active robots of the previous horizon (lines 11-12), it makes the inactive robots of the previous horizon $I^{inact}_{par}$ participants (lines 13-14). 
Then, it invokes \FnOnDemCPPHor (line 15), which does the path planning for the current horizon (we describe in Section \ref{subsec:cpp_horizon}). 
\FnOnDemCPPHor updates $I^{inact}_{par}$, which now contains the inactive robots of the current horizon, and returns the equal-length paths $\Sigma$ for the active robots $[R] \setminus I^{inact}_{par}$ of the current horizon. 
So, the CP simultaneously sends paths to those active robots through response messages $\mathtt{M_{res}}$ (lines 16 and \shortversion{19-22}\longversion{33-36}). 
Finally, it reinitializes some data structures for the next horizon (lines 17-18). 

%% file: algorithms/robot.tex
\begin{algorithm}[t]
    \small
    \DontPrintSemicolon
    
    \caption{$\mathtt{Robot}(i, s^i_0, X, Y)$}
    \label{algo:robot}
    
    $W^i \gets \mathtt{init\_localview}(s^i_0, X, Y)$\;
    $\mathtt{M_{req}[id, state, view]} \gets [i, s^i_0, W^i]$       \tcp*{Create $\mathtt{M_{req}}$}
    \While{true}{
        $\mathtt{send\_localview(M_{req})}$\;
        $\sigma^i \gets$ $\mathtt{receive\_path(M_{res})}$\;
        $W^i \gets \mathtt{follow\_path(\sigma^i)}$\;
        $\mathtt{M_{req}[state, view]} \gets [s^i_{|\sigma^i|}, W^i]$       \tcp*{Update $\mathtt{M_{req}}$}
    }
\end{algorithm}

%% file: algorithms/ondemcpp_v2.tex
\begin{algorithm}[t!]
    \small
    \DontPrintSemicolon
    
    \caption{$\mathtt{OnDemCPP}(R, X, Y)$}
    \label{algo:ondemcpp}
    

    $\Sigma_{rem} \gets \{\mathsf{NULL}\ |\ i \in [R]\}$      \tcp*{Init.}
    $I_{par} \gets \emptyset$, $S \gets \emptyset$, $I^{inact}_{par} \gets \emptyset$, $S^{inact} \gets \emptyset$, $W \gets \emptyset$, $\widetilde{W_g} \gets \emptyset$, $\Pi \gets \emptyset$\;
    $N_{req} \gets 0$       \tcp*{Number of requests}
    $N_{act} \gets R$       \tcp*{Number of active robots}    
    \BlankLine
    \Srv{\FRcvLV{$\mathtt{M_{req}}$}}{
        $i \gets \mathtt{M_{req}.id}$       \tcp*{Robot ID}
        
        \If{$\sigma^i_{rem} = \mathsf{NULL}$}{
            $I_{par} \gets I_{par} \cup \{i\}$      \tcp*{$r^i$ is a participant}
            $S \gets S \cup \{\mathtt{M_{req}.state}\}$\;
        }
        
        $\FnUpdateGlobalview(\mathtt{M_{req}.view})$\;
        $N_{req} \gets N_{req} + 1$\;

        \If{$(N_{req} = N_{act})$}{
            $I_{par} \gets I_{par} \cup I^{inact}_{par}$     \tcp*{Formerly inactive}
            $S \gets S \cup S^{inact}$\;
            $\Sigma \gets$ \FnOnDemCPPHor()      \tcp*{Eq.-length paths}
            \FnSndPathsToActRobs($\Sigma$)\;
            $I_{par} \gets \emptyset$, $S \gets \emptyset$      \tcp*{Reinit.}
            $N_{req} \gets 0$\;
        }
    }
    \longversion
    {
    \Fn{\FnUpdateGlobalview{$\mathtt{M_{req}.view}$}}{
        $W^i \gets \mathtt{M_{req}.view}$       \tcp*{Local view of $r^i$}
        \For{$(x, y) \in W^i_c$}{
            \If{$(x, y) \notin W_c$}{
                $W_c \gets W_c \cup \{(x, y)\}$     \tcp*{Covered}
            }
        }
        \For{$(x, y) \in W^i_g$}{
            \If{$(x, y) \in W_u$}{
                $W_g \gets W_g \cup \{(x, y)\}$     \tcp*{Goal}
            }
        }
        $W_g \gets W_g \setminus W_c$\;
        \For{$(x, y) \in W^i_o$}{
            \If{$(x, y) \in W_u$}{
                $W_o \gets W_o \cup \{(x, y)\}$     \tcp*{Obstacle}
            }
        }
        $W_o \gets W_o \setminus W_c$\;
        $W_u \gets ([X] \times [Y]) \setminus (W_c \cup W_g \cup W_o)$      \tcp*{Unexplored}
    }
    }
    \BlankLine
    \Fn{\FnSndPathsToActRobs{$\Sigma$}}{
        \ParFor{$i \in [R] \setminus I^{inact}_{par}$}{
              \ \ \ \ $\mathtt{M_{res}[path]} \gets \sigma^i$      \tcp*{Create $\mathtt{M_{res}}$}
              \ \ \ \ \ \FSndPath($\mathtt{M_{res}, i}$)\;
        }
    }
    \BlankLine
    \Fn{\FnOnDemCPPHor{}}{
        $\Sigma' \gets \emptyset$       \tcp*{Participants' coll.-free paths}
        
        \If{$(W_g\setminus\widetilde{W_g})\neq\emptyset$}{
            $\Sigma' \gets \FnCPPForPar(W, \widetilde{W_g}, I_{par}, S, M, R, \Sigma_{rem})$\;
        }
        \ElseIf{$I_{par} \neq [R]$}{
            \For{$i \in I_{par}$}{
                $\sigma'^i \gets s^i_0$\;
                $\Sigma' \gets \Sigma' \cup \{\sigma'^i\}$
            }
        }
        \Else{
            \Exit()     \tcp*{Coverage Complete !!!}
        }

        $\Sigma'_{all} \gets \mathtt{combine\_paths}(I_{par}, \Sigma', R, \Sigma_{rem})$\;
        $\lambda \gets \mathtt{determine\_horizon\_length}(\Sigma'_{all})$\;
        $\Sigma \gets$ \FGetEqLenPaths($\Sigma'_{all}, \lambda$)\;
        \Return $\Sigma$\;
    }
    \BlankLine
    \Fn{\FGetEqLenPaths{$\Sigma'_{all}, \lambda$}}{
        $\Sigma \gets \emptyset$, $\Sigma_{rem} \gets \emptyset$, $\widetilde{W_g} \gets \emptyset$, $I^{inact}_{par} \gets \emptyset$, $S^{inact} \gets \emptyset$\;
        $N_{act} \gets 0$       \tcp*{Reinit.}
    
        \For{$i \in [R]$}{
            \If{$|\sigma'^i_{all}| > 0$}{
                $N_{act} \gets N_{act} + 1$     \tcp*{$r^i$ is active}
                $\langle \sigma^i, \sigma^i_{rem} \rangle \gets \mathtt{split\_path}(\sigma'^i_{all}, \lambda)$\;
                $\Sigma \gets \Sigma \cup \{\sigma^i\}$\;

                \If{$|\sigma^i_{rem}| > 0$}{
                    $\widetilde{W_g} \gets \widetilde{W_g} \cup \{\mathcal{L}(s^i_{|\sigma'^i_{all}|})\}$       \tcp*{Reserve}
                }
                \Else{
                    $\sigma^i_{rem} \gets \mathsf{NULL}$\;
                }
                
                $\pi^i \gets \pi^i : \sigma^i$      \tcp*{Concat.}
            }
            \Else{
                $I^{inact}_{par} \gets I^{inact}_{par} \cup \{i\}$      \tcp*{$r^i$ is inactive}
                $S^{inact} \gets S^{inact} \cup \{s^i_0\}$\;
                $\sigma^i_{rem} \gets \mathsf{NULL}$\;
                $\pi^i \gets \pi^i : \mathtt{dummy\_path}(s^i_0, \lambda)$      \tcp*{Concat.}
            }
            
            $\Sigma_{rem} \gets \Sigma_{rem} \cup \{\sigma^i_{rem}\}$\;
        }

        \Return $\Sigma$
    }
\end{algorithm}

%% file: algorithms/ondemcpp.tex
\begin{algorithm*}[t!]
    \small
    \DontPrintSemicolon
    
    \caption{$\mathtt{OnDemCPP}(R, X, Y)$}
    \label{algo:ondemcpp}
    
    \begin{multicols}{2}

    $\Sigma_{rem} \gets \{\mathsf{NULL}\ |\ i \in [R]\}$      \tcp*{Init.}
    $I_{par} \gets \emptyset$, $S \gets \emptyset$, $I^{inact}_{par} \gets \emptyset$, $S^{inact} \gets \emptyset$, $W \gets \emptyset$, $\widetilde{W_g} \gets \emptyset$, $\Pi \gets \emptyset$\;
    $N_{req} \gets 0$       \tcp*{Number of requests}
    $N_{act} \gets R$       \tcp*{Number of active robots}    
    \BlankLine
    \Srv{\FRcvLV{$\mathtt{M_{req}}$}}{
        $i \gets \mathtt{M_{req}.id}$       \tcp*{Robot ID}
        
        \If{$\sigma^i_{rem} = \mathsf{NULL}$}{
            $I_{par} \gets I_{par} \cup \{i\}$      \tcp*{$r^i$ is a participant}
            $S \gets S \cup \{\mathtt{M_{req}.state}\}$\;
        }
        
        $\FnUpdateGlobalview(\mathtt{M_{req}.view})$\;
        $N_{req} \gets N_{req} + 1$\;

        \If{$(N_{req} = N_{act})$}{
            $I_{par} \gets I_{par} \cup I^{inact}_{par}$     \tcp*{Formerly inactive}
            $S \gets S \cup S^{inact}$\;
            $\Sigma \gets$ \FnOnDemCPPHor()      \tcp*{Equal-length paths}
            \FnSndPathsToActRobs($\Sigma$)\;
            $I_{par} \gets \emptyset$, $S \gets \emptyset$      \tcp*{Reinit.}
            $N_{req} \gets 0$\;
        }
    }
    \BlankLine
    \Fn{\FnUpdateGlobalview{$\mathtt{M_{req}.view}$}}{
        $W^i \gets \mathtt{M_{req}.view}$       \tcp*{Local view of $r^i$}
        \For{$(x, y) \in W^i_c$}{
            \If{$(x, y) \notin W_c$}{
                $W_c \gets W_c \cup \{(x, y)\}$     \tcp*{Covered}
            }
        }
        \For{$(x, y) \in W^i_g$}{
            \If{$(x, y) \in W_u$}{
                $W_g \gets W_g \cup \{(x, y)\}$     \tcp*{Goal}
            }
        }
        $W_g \gets W_g \setminus W_c$\;
        \For{$(x, y) \in W^i_o$}{
            \If{$(x, y) \in W_u$}{
                $W_o \gets W_o \cup \{(x, y)\}$     \tcp*{Obstacle}
            }
        }
        $W_o \gets W_o \setminus W_c$\;
        $W_u \gets ([X] \times [Y]) \setminus (W_c \cup W_g \cup W_o)$      \tcp*{Unexplored}
    }
    \BlankLine
    \Fn{\FnSndPathsToActRobs{$\Sigma$}}{
        \ParFor{$i \in [R] \setminus I^{inact}_{par}$}{
              \ \ \ \ $\mathtt{M_{res}[path]} \gets \sigma^i$      \tcp*{Create $\mathtt{M_{res}}$}
              \ \ \ \ \ \FSndPath($\mathtt{M_{res}, i}$)\;
        }
    }
    \BlankLine
    \Fn{\FnOnDemCPPHor{}}{
        $\Sigma' \gets \emptyset$       \tcp*{Participants' collision-free paths}
        
        \If{$(W_g\setminus\widetilde{W_g})\neq\emptyset$}{
            $\Sigma' \gets \FnCPPForPar(W, \widetilde{W_g}, I_{par}, S, M, R, \Sigma_{rem})$\;
        }
        \ElseIf{$I_{par} \neq [R]$}{
            \For{$i \in I_{par}$}{
                $\sigma'^i \gets s^i_0$\;
                $\Sigma' \gets \Sigma' \cup \{\sigma'^i\}$
            }
        }
        \Else{
            \Exit()     \tcp*{Coverage Complete !!!}
        }

        $\Sigma'_{all} \gets \mathtt{combine\_paths}(I_{par}, \Sigma', R, \Sigma_{rem})$\;
        $\lambda \gets \mathtt{determine\_horizon\_length}(\Sigma'_{all})$\;
        $\Sigma \gets$ \FGetEqLenPaths($\Sigma'_{all}, \lambda$)\;
        \Return $\Sigma$\;
    }
    \BlankLine
    \Fn{\FGetEqLenPaths{$\Sigma'_{all}, \lambda$}}{
        $\Sigma \gets \emptyset$, $\Sigma_{rem} \gets \emptyset$, $\widetilde{W_g} \gets \emptyset$, $I^{inact}_{par} \gets \emptyset$, $S^{inact} \gets \emptyset$\;
        $N_{act} \gets 0$       \tcp*{Reinit.}
    
        \For{$i \in [R]$}{
            \If{$|\sigma'^i_{all}| > 0$}{
                $N_{act} \gets N_{act} + 1$     \tcp*{$r^i$ is active}
                $\langle \sigma^i, \sigma^i_{rem} \rangle \gets \mathtt{split\_path}(\sigma'^i_{all}, \lambda)$\;
                $\Sigma \gets \Sigma \cup \{\sigma^i\}$\;

                \If{$|\sigma^i_{rem}| > 0$}{
                    $\widetilde{W_g} \gets \widetilde{W_g} \cup \{\mathcal{L}(s^i_{|\sigma'^i_{all}|})\}$       \tcp*{Reserve}
                }
                \Else{
                    $\sigma^i_{rem} \gets \mathsf{NULL}$\;
                }
                
                $\pi^i \gets \pi^i : \sigma^i$      \tcp*{Concat.}
            }
            \Else{
                $I^{inact}_{par} \gets I^{inact}_{par} \cup \{i\}$      \tcp*{$r^i$ is inactive}
                $S^{inact} \gets S^{inact} \cup \{s^i_0\}$\;
                $\sigma^i_{rem} \gets \mathsf{NULL}$\;
                $\pi^i \gets \pi^i : \mathtt{dummy\_path}(s^i_0, \lambda)$      \tcp*{Concat.}
            }
            
            $\Sigma_{rem} \gets \Sigma_{rem} \cup \{\sigma^i_{rem}\}$\;
        }

        \Return $\Sigma$
    }
    \end{multicols}
\end{algorithm*}

%% file: 5c_horizon.tex
\subsection{Coverage Path Planning in the Current Horizon}
\label{subsec:cpp_horizon}

Here, we describe the rest of Algorithm \ref{algo:ondemcpp}. 
Recall that the CP has already assigned the goals $\widetilde{W_g}$ to the non-participants in some past horizons. 
In \FnOnDemCPPHor (lines \shortversion{23-36}\longversion{37-50}), first, the CP checks whether there are \textit{unassigned} goals, i.e., $W_g \setminus \widetilde{W_g}$ left in the workspace (line \shortversion{25}\longversion{39}). 
If so, it attempts to generate the collision-free paths $\Sigma' = \{\sigma'^i\ |\ i \in I_{par}\}$ for the participants $I_{par}$ to visit some unassigned goals while keeping the remaining paths $\Sigma_{rem}$ of the non-participants $[R] \setminus I_{par}$ (hereafter denoted by $\overline{I_{par}}$) \textbf{intact} (line \shortversion{26}\longversion{40}). 
We defer the description of \FnCPPForPar to Section~\ref{subsec:cpp_for_par}. 
Otherwise, the CP cannot plan for the participants. 
Thereby, it \textit{skips} replanning the participants in the current horizon. 
Still, if there are non-participants (line \shortversion{27}\longversion{41}), they can proceed toward their assigned goals in the current horizon while the participants remain inactive. 
To signify this, the CP assigns the current states of the participants to their paths $\Sigma'$ (lines \shortversion{28-30}\longversion{42-44}). 
\shortversion{Please see Example 4 in \cite{DBLP:conf/arxiv/MitraS23}, where we show such a scenario. }
Note that when both the criteria, viz all the robots are participants, and there are no unassigned goals get fulfilled (lines \shortversion{31-32}\longversion{45-46}), it means complete coverage (established in Theorem \ref{theorem:asyncpp_complete_coverage} in Section \ref{sec:theoretical_analysis}). 
Next, the CP combines $\Sigma'$ with $\Sigma_{rem}$ into $\Sigma'_{all} = \{\sigma'^i_{all} | i \in [R]\}$, where 
$$
\sigma'^i_{all}=
    \begin{cases}
        \sigma'^i, & \text{if $i \in I_{par}$},\\
        \sigma^i_{rem}, & \text{otherwise}.
    \end{cases}
$$
Thus, $\sigma'^i_{all}$ contains the collision-free path $\sigma'^i$ if $r^i$ is a participant, otherwise contains the remaining path $\sigma^i_{rem}$ of the non-participant $r^i$. 
Notice that $\Sigma'_{all}$ contains the collision-free paths for all the robots (line \shortversion{33}\longversion{47}). 
A robot $r^i$ is said to be active in the current horizon if its path length is non-zero, i.e., $|\sigma'^i_{all}| > 0$. 
Otherwise, $r^i$ is said to be inactive. 
In the penultimate step, the CP determines the length of the current horizon (line \shortversion{34}\longversion{48}) as the minimum path length of the active robots, i.e., $\lambda = \min_{i \in [R]} \{|\sigma'^i_{all}| > 0\}$. 
Finally, it returns $\lambda$-length paths $\Sigma$ only for the active robots (lines \shortversion{35-36}\longversion{49-50}). 
This way, the CP ensures that at least one active robot (participant or non-participant) reaches its goal in the current horizon. 

In $\mathtt{get\_equal\_length\_paths}$ (lines \shortversion{37-56}\longversion{51-70}), after re-initialization of some of the data structures (lines \shortversion{38-39}\longversion{52-53}), the CP examines the paths $\Sigma'_{all}$ in a \textbf{for} loop (lines \shortversion{40-55}\longversion{54-69}) to determine which robots are active (lines \shortversion{41-49}\longversion{55-63}) and which are not (lines \shortversion{50-54}\longversion{64-68}). 
If $r^i$ is active, first, the CP increments $N_{act}$ by $1$ (line \shortversion{42}\longversion{56}). 
Then, the CP splits its path $\sigma'^i_{all}$ into two parts $\sigma^i$ and $\sigma^i_{rem}$ of lengths $\lambda$ and $|\sigma'^i_{all}| - \lambda$, respectively (line \shortversion{43}\longversion{57}). 
Active $r^i$ traverses the former part $\sigma^i$, containing $s^i_0 \ldots s^i_{\lambda}$, in the current horizon (line \shortversion{44}\longversion{58}) and the later part (if any), containing $s^i_{\lambda} \ldots s^i_{|\sigma'^i_{all}|}$, in the future horizons (line \shortversion{55}\longversion{69}). 
If it has the remaining path $\sigma^i_{rem}$, its goal location $\mathcal{L}(s^i_{|\sigma'^i_{all}|})$ gets added to $\widetilde{W_g}$ (lines \shortversion{45-46}\longversion{59-60}), signifying that this goal remains reserved for $r^i$. 
Otherwise, $\sigma^i_{rem}$ is set to $\mathsf{NULL}$ (lines \shortversion{47-48}\longversion{61-62} and \shortversion{55}\longversion{69}), indicating the absence of the remaining path for $r^i$, thereby enabling $r^i$ to become a participant in the next horizon. 
In contrast, if $r^i$ is inactive (lines \shortversion{50-54}\longversion{64-68}), the CP adds $r^i$'s ID $\{i\}$ and its start state $s^i_0$ into $I^{inact}_{par}$ and $S^{inact}$, respectively (lines \shortversion{51-52}\longversion{65-66}), as it would reattempt to find $r^i$'s path in the next horizon. 
The CP also sets $\sigma^i_{rem}$ to $\mathsf{NULL}$ (lines \shortversion{53}\longversion{67} and \shortversion{55}\longversion{69}). 
Note that the CP also generates the full path for the active robots (line \shortversion{49}\longversion{63}) and the inactive robots (line \shortversion{54}\longversion{68}), where $\mathtt{dummy\_paths}$ generates a path of length $\lambda$ containing only the current state. 
The CP sends no path to an inactive robot in the current horizon (lines \shortversion{20-22}\longversion{34-36}). 
So, the active robots of the current horizon send requests in the next horizon (lines \shortversion{42}\longversion{56} and 11-12). 

\longversion
{
\begin{example}
\label{ex:skip_plan_par}
In Figure \ref{fig:skip_replan_par}, we show an example (continuation of Example \ref{ex:update_gv}) where the CP skips planning for a participant due to the unavailability of an unassigned goal. 
In horizon $1$, both the robots $r^1$ and $r^2$ participate, and the CP finds their collision-free paths $\sigma'^1$ and $\sigma'^2$ of lengths $2$ and $1$, respectively. 
Note that both robots are active. 
As the horizon length ($\lambda$) is found to be $1$, the CP reserves $r^1$'s goal into $\widetilde{W_g}$, and sets its remaining path $\sigma^1_{rem}$ accordingly. 
Each robot follows its $1$-length path and sends the updated local view to the CP. 
In horizon $2$, though the CP receives $2$ requests, only $r^2$ participates. 
As no unassigned goal exists (i.e., $W_g \setminus \widetilde{W_g} = \emptyset$), the CP fails to find any path for $r^2$. 
So, $r^2$ becomes inactive, and the CP adds its ID into $I^{inact}_{par}$ and state into $S^{inact}$ as it would reattempt in horizon $3$. 
Now, only $r^1$ (non-participant) is active. 
As $\lambda$ is found to be $1$, $r^1$ visits its reserved goal, which the CP removes from $\widetilde{W_g}$ and resets $\sigma^1_{rem}$. 
Notice that in horizon $3$, the CP receives only $1$ request (from $r^1$ only), but both the robots participate. 
\end{example}
}

%% file: 5d_participants_short.tex
\subsection{Coverage Path Planning for the Participants}
\label{subsec:cpp_for_par}

At last, we give the outline of \FnCPPForPar (Algorithm \ref{algo:cpp_for_par}), which attempts to generate collision-free paths only for the participants without modifying the remaining paths of the non-participants. 
The in-depth explanation is in \cite{DBLP:conf/arxiv/MitraS23}. 
It is a modified version of Algorithms 2 and 3 of \cite{DBLP:conf/iros/MitraS22} combined, which uses the idea of \textit{Goal Assignment-based Prioritized Planning} to generate collision-free paths for all the robots in each horizon. 
Put differently, all the robots participate in each horizon of \cite{DBLP:conf/iros/MitraS22}. 

\begin{algorithm}[t]
    \small
    \DontPrintSemicolon
    
    \caption{$\mathtt{CPPForPar}(W, \widetilde{W_g}, I_{par}, S, M, R, \Sigma_{rem})$}
    \label{algo:cpp_for_par}
    
    \KwResult{Collision-free paths for the participants $(\Sigma')$}
    
    $\langle \Gamma,\ \Phi \rangle \gets \mathtt{COPForPar}(W, \widetilde{W_g}, I_{par}, S, M)$\;
    $\Sigma' \gets \mathtt{CFPForPar}(\Gamma, \Phi, I_{par}, R, \Sigma_{rem})$\;
\end{algorithm}

\smallskip
\subsubsection{Sum-of-Costs-optimal goal assignment and corresponding paths}

Let $R^*$ and $G^*$ be the number of participants and unassigned goals in the current horizon, respectively. 
So, $R^* = |I_{par}| = |S|$ and $G^* = |W_g \setminus \widetilde{W_g}|$. 
The CP uses the Hungarian Algorithm \cite{kuhn1955hungarian} to assign the participants $I_{par}$ to the unassigned goals having IDs [$G^*$], providing which participant would go to which goal (line 1). 
Formally, the assignment array is $\Gamma: I_{par} \rightarrow [G^*] \cup \{\mathsf{NULL}\}$. 
For a participant $r^i$, where $i \in I_{par}$, its assigned goal $\Gamma[i]$ contains $\mathsf{NULL}$ if the CP fails to assign any goal to $r^i$, e.g., when $R^* > G^*$. 
Such a participant is said to be \emph{inactive}, whose corresponding optimal path $\varphi^i$ only contains its current state $s^i_0$. 
Otherwise, the participant is said to be \emph{active}, whose corresponding optimal path $\varphi^i$ leads $r^i$ from its current state $s^i_0$ to the assigned goal $\Gamma[i]$ using motions $M$. 
Such an optimal path $\varphi^i$ passes through the cells $W_g \cup W_c$ but avoids the cells $W_u \cup W_o$. 
We denote the resultant optimal paths by $\Phi = \{\varphi^i\ |\ i \in I_{par}\}$. 
Keep in mind that the non-participants are inherently active because $\sigma^j_{rem} \neq \mathsf{NULL},\ \forall j \in \overline{I_{par}}$. 

\smallskip
\subsubsection{Collision-free paths}
\label{subsubsec:cfp_for_par}

The participants' optimal paths $\Phi$ are not necessarily \textit{inter-robot collision-free}. 
A participant $r^i$ may collide with either another participant $r^j$ or a non-participant $r^k$ having remaining path $\sigma^k_{rem}$. 
Though the CP cannot change the remaining paths $\Sigma_{rem}$ of the non-participants, it can change the paths $\Phi$ for the participants. 
To get rid of any collisions, the CP employs a two-step procedure. 
First, it \textit{prioritizes} the participants \textit{dynamically} based on their movement constraints in $\Phi$. 
For example, if a participant $r^i$'s start location is on the path $\varphi^j$ of another participant $r^j$, then $r^i$ must leave its start location before $r^j$ reaches there. 
Similarly, if $r^i$'s goal is on the path $\varphi^j$ of $r^j$, then $r^j$ must get past that goal before $r^i$ reaches there. 
The non-participants have higher priorities than the participants because the CP cannot change their remaining paths. 
Finally, the CP \textit{offsets} the participants' paths $\Phi$ in the order of their priority to ensure that a participant does not collide with its higher priority robots (some other participants and the non-participants). 
However, if a collision between a participant and a non-participant becomes inevitable (Example \shortversion{5 in \cite{DBLP:conf/arxiv/MitraS23}}\longversion{\ref{ex:sto}} shows such a scenario), the CP inactivates the participant and returns to the first step to recompute the priorities. 
Otherwise, the CP returns the participants' collision-free paths $\Sigma'$. 
Note that all the participants may get inactivated in the worst case during offsetting due to the non-participants. 


\longversion
{
\begin{figure}
    \centering
    \includegraphics[scale=0.25]{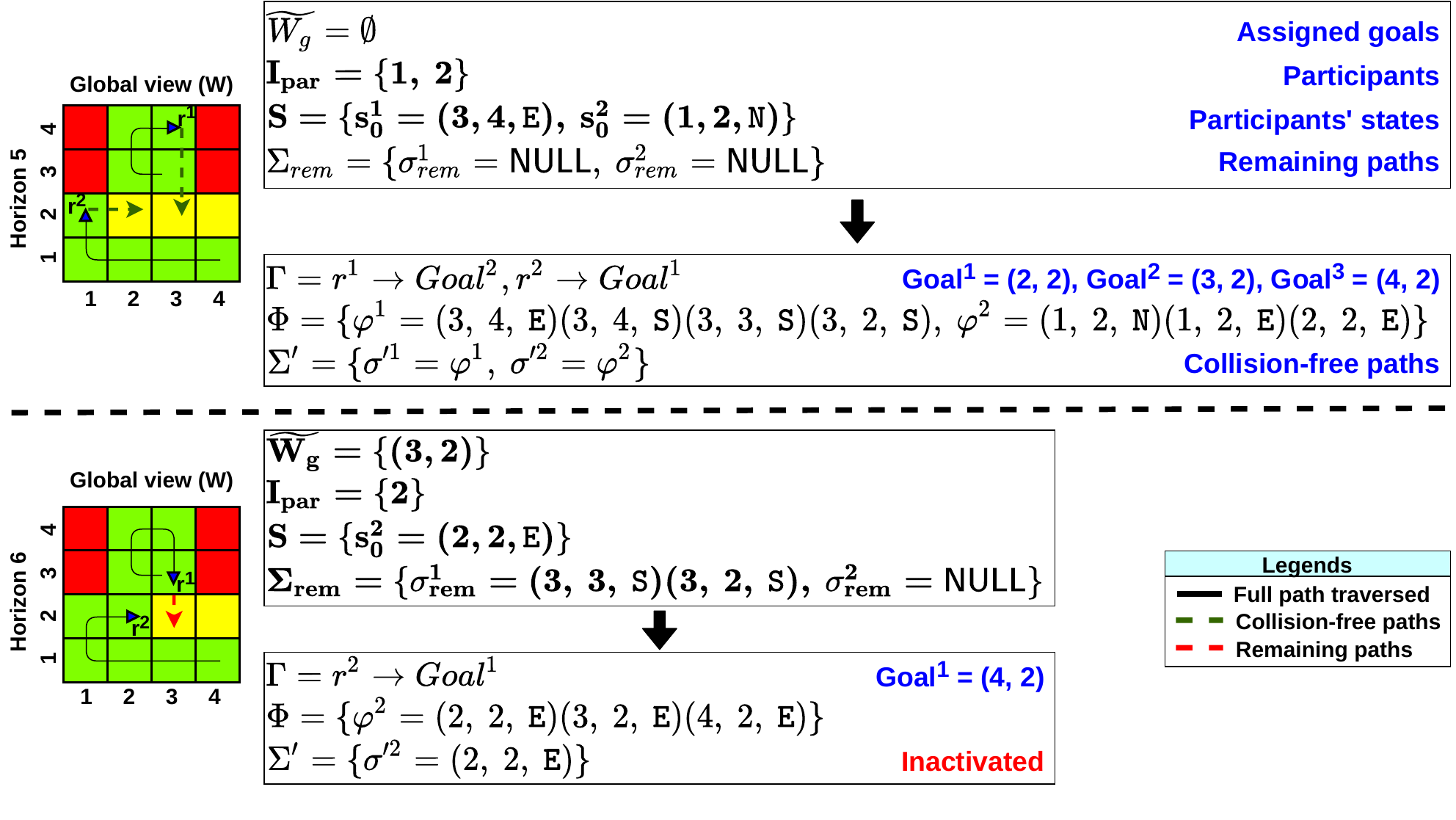}
    \caption{Inactivation of a participant during offsetting}
    \label{fig:sto}
\end{figure}

\begin{example}
\label{ex:sto}
In Figure \ref{fig:sto}, we show an extreme case of replanning where the CP inactivates a participant to avoid collision with a non-participant. 
Notice that both the robots $r^1$ and $r^2$ have traversed their full paths of length five till horizon $4$. 
In horizon $5$, both participate because none has any remaining path. 
The CP assigns the unassigned goal IDs $2$ and $1$ to $r^1$ and $r^2$, respectively, and subsequently finds their collision-free paths $\Sigma'$. 
As the horizon length ($\lambda$) is found to be $2$, the CP reserves the goal $(3,\ 2)$ for $r^1$ in $\widetilde{W_g}$ and sets its remaining path $\sigma^1_{rem}$ accordingly. 
Therefore, only $r^2$ participates in horizon $6$. 
The CP finds $r^2$'s optimal path $\varphi^2$ leading to the only unassigned goal $(4,\ 2)$. 
Though the prioritization succeeds trivially, the offsetting fails because the non-participant $r^1$ would reach its goal $(3,\ 2)$ before the participant $r^2$ gets past that goal. 
So, the CP inactivates $r^2$. 
In this horizon, only $r^1$ (non-participant) remains active with $\lambda = 1$. 
Observe that both will participate in the next horizon. 
\end{example}
}

%% file: 5d_participants.tex
\subsection{Coverage Path Planning for the Participants}
\label{subsec:cpp_for_par}

In detail, we now present \FnCPPForPar (Algorithm~\ref{algo:cpp_for_participants}), which in every invocation, generates the collision-free paths only for the participants without altering the remaining paths for the non-participants. 
It is a modified version of Algorithms 2 and 3 of \cite{DBLP:conf/iros/MitraS22} combined, which uses the idea of \textit{Goal Assignment-based Prioritized Planning} to generate collision-free paths for all the robots. 
Let $R^*$ and $G^*$ be the number of participants and unassigned goals in the current horizon. 
So, $R^* = |I_{par}| = |S|$ and $G^* = |W_g \setminus \widetilde{W_g}|$. 
The problem of optimally visiting $G^*$ unassigned goals with $R^*$ participants is a \textit{Multiple Traveling Salesman Problem}~\cite{DBLP:conf/amcc/OberlinRD09}, which is \textit{NP-hard}. 
Hence, it is computationally costly when $R^*$, $G^*$, or both are large. 
So, the CP generates the participants' paths in two steps without guaranteeing optimality. 
First, it assigns the participants to the unassigned goals in a cost-optimal way and finds their optimal paths (line 1 and Algorithm \ref{algo:cop_for_par}). 
Then, if needed, it makes those optimal paths collision-free by adjusting them using a greedy method (line 2 and Algorithm \ref{algo:cfp_for_par}), thereby may lose optimality. 

\begin{algorithm}[t]
    \small
    \DontPrintSemicolon
    
    \caption{$\mathtt{CPPForPar}(W, \widetilde{W_g}, I_{par}, S, M, R, \Sigma_{rem})$}
    \label{algo:cpp_for_participants}
    
    \KwResult{Collision-free paths for the participants $(\Sigma')$}
    
    $\langle \Gamma,\ \Phi \rangle \gets \FnCOPForPar(W, \widetilde{W_g}, I_{par}, S, M)$\;
    $\Sigma' \gets \FnCFPForPar(\Gamma, \Phi, I_{par}, R, \Sigma_{rem})$\;
\end{algorithm}

\begin{algorithm}[t]
    \small
    \DontPrintSemicolon
    
    \caption{$\mathtt{COPForPar}(W, \widetilde{W_g}, I_{par}, S, M)$}
    \label{algo:cop_for_par}
    
    \KwResult{Opt. assignment ($\Gamma$) and corresponding paths $(\Phi)$}
    
    $\langle \Delta, L_{\Delta} \rangle \gets \mathtt{compute\_optimal\_costs}(W, \widetilde{W_g}, I_{par}, S, M)$\;
    $\Gamma \gets \mathtt{compute\_optimal\_assignments}(\Delta, I_{par})$\;
    $\Phi \gets \mathtt{get\_optimal\_paths}(L_{\Delta}, \Gamma, I_{par})$\;
\end{algorithm}

\smallskip
\subsubsection{Sum-of-Costs-optimal goal assignment and corresponding paths}

First, we construct a \textit{weighted state transition graph} $\mathcal{G}_\Delta$. 
Its vertices $\mathcal{G}_\Delta.V$ are the all possible states of the participants s.t. $\mathcal{L}(s) \in W_g \cup W_c$, $\forall s \in \mathcal{G}_\Delta.V$. 
And its edges $\mathcal{G}_\Delta.E$ are the all possible transitions among the states, i.e., $e(s', s'') \in \mathcal{G}_\Delta.E$ iff $s', s'' \in \mathcal{G}_\Delta.V$ and $\exists \mu \in M$ s.t. $s' \xrightarrow{\mu}s''$. 
Moreover, the weight of an edge $e$ is $\mathtt{cost}(\mu)$. 
Now, for each pair of participant $r^i$ (i.e., $i \in I_{par}$) and unassigned goal $\gamma^j \in W_g \setminus \widetilde{W_g}$ (i.e., $j \in [G^*]$), we run the A* algorithm \cite{DBLP:books/cu/L2006, DBLP:books/daglib/0023820} to compute the optimal path that originates from the start state $s^i_0$ and terminates at the state $s$ s.t. $\mathcal{L}(s) = \gamma^j$. 
We store the cost of the path in $\Delta[i][j] \in \mathbb{R}$, and the corresponding path in $L_\Delta[i][j]$ (line 1 of Algorithm \ref{algo:cop_for_par}). 
Next, we apply the Hungarian algorithm \cite{kuhn1955hungarian, DBLP:journals/cacm/BourgeoisL71} on $\Delta$ to find the cost-optimal assignment $\Gamma: I_{par} \rightarrow [G^*] \cup \{\mathsf{NULL}\}$ (line 2), signifying which participant would go to which unassigned goal. 
Notice that some participants would get assigned $\mathsf{NULL}$ when $R^* > G^*$. 
We call such participants inactive and the rest of the participants active. 
Finally, for each participant $r^i$, its optimal path $\varphi^i \in \Phi$ is set to $L_\Delta[i][j]$ if $\Gamma[i]=j$, otherwise (i.e., when participant $r^i$ is inactive), set to its start state $s^i_0$ (line 3). 
Thus, we get the set of optimal paths for the participants $\Phi = \{\varphi^i\ |\ i \in I_{par}\}$. 
Keep in mind that the non-participants $\overline{I_{par}}$ are inherently active as their remaining paths $\sigma^k_{rem} \neq \mathsf{NULL}, \forall k \in \overline{I_{par}}$. 


\begin{figure}[tb]
    \centering
    \includegraphics[scale=0.33]{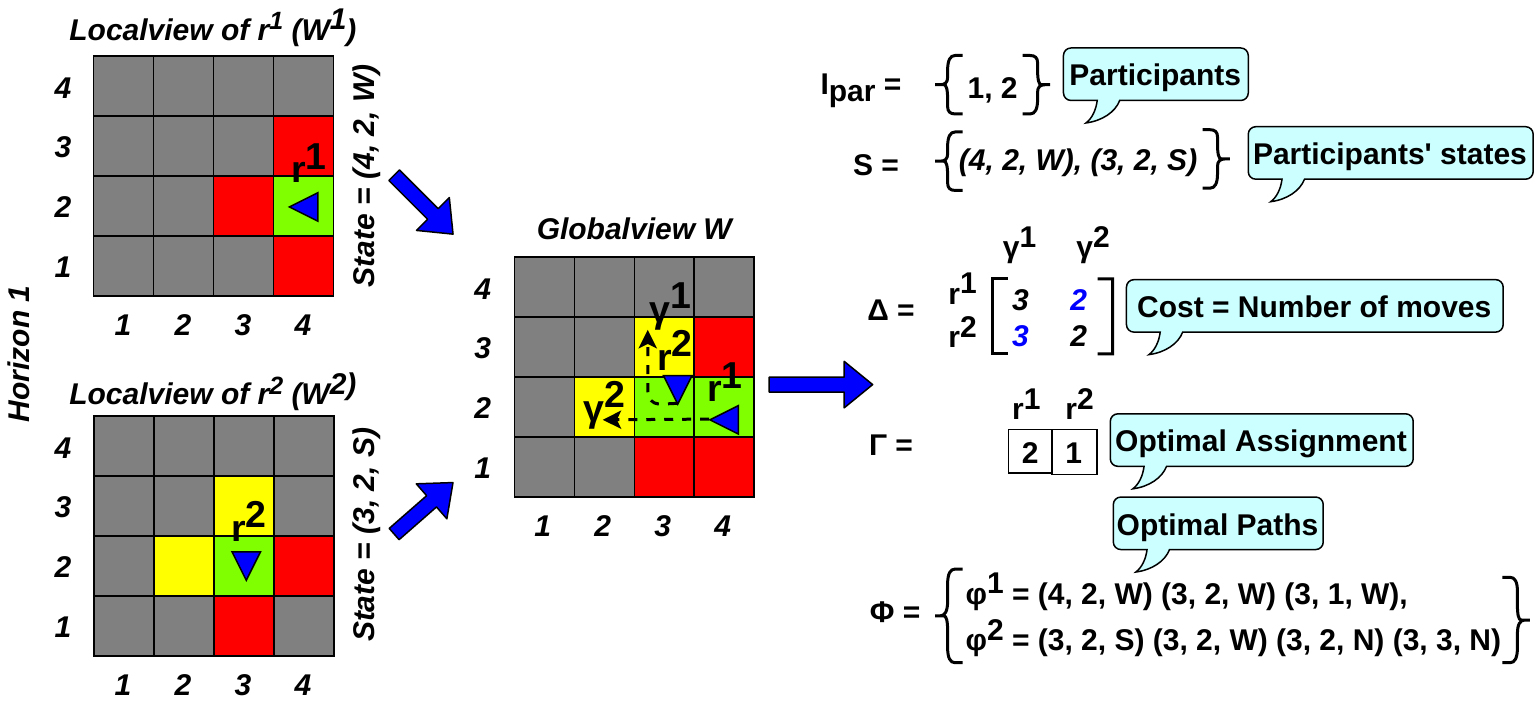}
    \caption{Cost-optimal goal assignment and optimal paths}
    \label{fig:assignment_and_paths}
\end{figure}

\begin{example}
\label{ex:assignment_and_paths}
We show an example of a cost-optimal goal assignment and corresponding optimal paths in Figure \ref{fig:assignment_and_paths} involving two participants $r^1$ and $r^2$ in a $4 \times 4$ workspace. 
\end{example}

\begin{algorithm}[t]
    \small
    \DontPrintSemicolon
    
    \caption{$\mathtt{CFPForPar}(\Gamma, \Phi, I_{par}, R, \Sigma_{rem})$}
    \label{algo:cfp_for_par}
    
    \KwResult{Collision-free paths for the participants $(\Sigma')$}
    
    \While{$true$}{
        $\langle \Gamma, \Omega \rangle \gets \mathtt{get\_feasible\_paths}(\Gamma, \Phi, I_{par})$\;
        $\Theta_r \gets \mathtt{compute\_relative\_precedences}(\Omega, I_{par})$\;
        $\Theta_a \gets \mathtt{compute\_absolute\_precedence}(\Theta_r, I_{par})$\;
        
        \If{$\Theta_a$ is valid}{
            $\Gamma' = \Gamma$     \tcp*{Itemwise copy}
            $\langle \Gamma, \Upsilon \rangle \gets \mathtt{compute\_sto}(\Gamma, \Omega, \Theta_a, I_{par}, R, \Sigma_{rem})$\;
            
            \lIf{$\Gamma' = \Gamma$}{
                \Break
            }
        }
        \Else{
            $\Gamma \gets \mathtt{break\_precedence\_cycles}(\Gamma, \Theta_r, I_{par})$\;
        }
        
        $\Phi \gets \mathtt{adjust\_paths}(\Gamma, \Omega, I_{par})$\;
    }
    
    $\Sigma' \gets \mathtt{get\_collision\_free\_paths}(\Omega, \Upsilon, I_{par})$\;
\end{algorithm}

\smallskip
\subsubsection{Prioritization}

Prioritization is an \textit{inter-robot collision-avoidance} mechanism, which assigns priorities to the robots either \textit{statically} (e.g., \cite{DBLP:conf/icra/ErdmannL86, DBLP:conf/aiide/Silver05}) or \textit{dynamically} (e.g., \cite{DBLP:conf/iros/BergO05, DBLP:conf/rss/TurpinMMK13}) so that a CP can subsequently find their collision-free paths in order of their priority. 
In \FnCFPForPar (Algorithm \ref{algo:cfp_for_par}), prioritization of the participants $I_{par}$ takes place (in the \textbf{while} loop in lines 1-11 except the \textbf{if} block in lines 5-8) based on the movement constraints among their paths in $\Phi$. 
But, prioritization surely fails if there is any \textit{infeasible} path in $\Phi$. 
So, detection and correction of the infeasible paths are necessary (line 2). 
There are two types of infeasible paths, namely \textit{crossover paths} and \textit{nested paths}. 
Optimal paths $\varphi^i$ and $\varphi^j$, where $\varphi^i,\ \varphi^j \in \Phi$, of distinct participants $r^i$ and $r^j$, where $i,\ j \in I_{par}$, are said to be - 

\begin{enumerate}
    \item a \textit{crossover path pair} if one of the participants is inactive, which sits on the path of the other active participant or their start locations are on each other's path, and 
    \item a \textit{nested path pair} if the active participant $r^i$'s start and goal locations are on $\varphi^j$ of the active participant $r^j$, or vice versa. 
\end{enumerate}

\begin{figure}[t]
    \centering
    \includegraphics[scale=0.37]{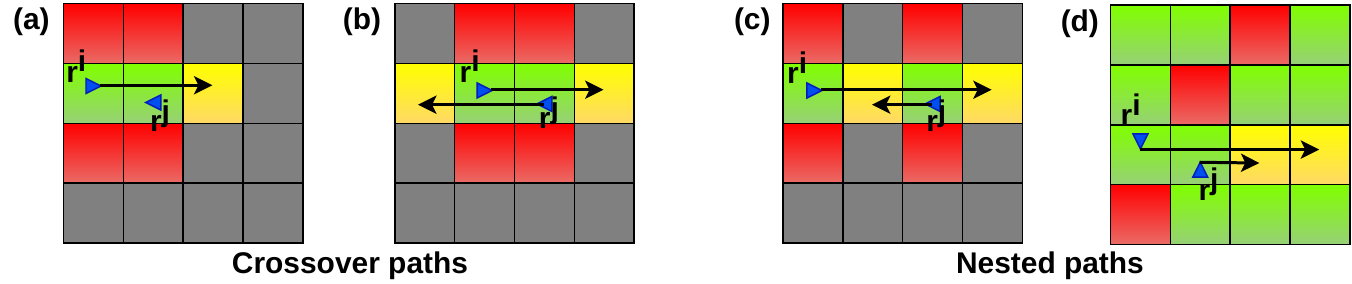}
    \caption{Infeasible paths}
    \label{fig:infeasible_paths}
\end{figure}

\begin{example}
\label{ex:infeasible_paths}
We show four simple examples of infeasible paths involving two participants $r^i$ and $r^j$ in Figure ~\ref{fig:infeasible_paths}. 
\end{example}

The CP can modify the paths $\Phi$ of the participants, but not the remaining paths $\Sigma_{rem}$ of the non-participants. 
So, first, we mark those active participants as \textit{killed}, which form infeasible path pairs with an inactive, active, or already killed participant. 
In the revival step, we initially mark all the inactive and the killed participants as \textit{unrevived}. 
Then, each path $\varphi^i$ of the killed participants is processed in reverse until an unrevived participant $r^j$ gets found. 
If $r^j$ can reach $r^i$'s goal by following $\omega^j \in \Omega$ (note that $\omega^j$ gets generated from $\varphi^i$) without forming any infeasible path pair, then $r^j$ gets \textit{revived}. 
Thus, a greedy procedure (a modified version of Algorithm 4 in \cite{DBLP:conf/iros/MitraS22}, which works only for $R$ participants) generates the feasible paths $\Omega = \{\omega^i\ |\ i\ \in\ I_{par}\}$ for the participants through \textit{goal reassignments}. 

\begin{example}
As a continuation of Example \ref{ex:infeasible_paths}, in Figure ~\ref{fig:infeasible_paths}, only $r^i$ is an active participant in (a) while both $r^i$ and $r^j$ are active participants in the rest (b-d). 
Due to infeasible paths, only $r^i$ gets killed in (a) while both $r^i$ and $r^j$ get killed in the rest (b-d). 
Though only $r^j$ gets revived in (a, c-d), both $r^i$ and $r^j$ get revived in (b). 
\end{example}

Next, for each pair of distinct participants $r^i$ and $r^j$, we compute the relative precedence $\Theta_r[i][j] \in \mathbb{B}$ (line 3) based on their paths $\omega^i$ and $\omega^j$, respectively, where $\omega^i, \omega^j \in \Omega$. 
$\Theta_r[i][j]$ is set to $1$ if $r^i$ must move before $r^j$, which is necessary when either the start location of $r^i$ is on $\omega^j$, or the goal location of $r^j$ is on $\omega^i$, $0$ otherwise. 
Similarly, we compute $\Theta_r[j][i]$. 
Finally, we apply the topological sorting on $\Theta_r$ to get the absolute precedence $\Theta_a$ (line 4). 
But, it fails if there is any cycle of relative precedence (see Example \ref{ex:cyclic_precedence}). 
In that case, we use \cite{DBLP:journals/ipl/HassinR94, DBLP:books/daglib/0022194} to get a Directed Acyclic Graph (DAG) out of that Directed Cyclic Graph (DCG) by inactivating some active participants in $\Gamma$ and adjusting their feasible paths $\Omega$ accordingly (lines 9-11), which in turn may form new crossover paths. 
Therefore, we reexamine the paths in another iteration of the \textbf{while} loop (line 1). 

\begin{figure}[tb]
    \centering
    \includegraphics[scale=0.36]{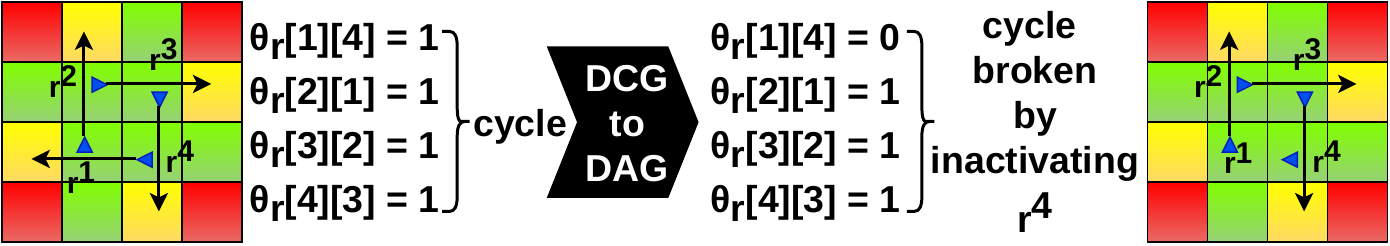}
    \caption{Removal of cyclic precedence}
    \label{fig:cyclic_precedence}
\end{figure}

\begin{example}
\label{ex:cyclic_precedence}
In Figure ~\ref{fig:cyclic_precedence}, we show an example of removing a precedence cycle involving four participants. 
\end{example}

\begin{example}
\label{ex:prioritization}
As a continuation of Example \ref{ex:assignment_and_paths}, the feasible paths $\omega^1$ and $\omega^2$ for the participants $r^1$ and $r^2$ are the same as $\varphi^1$ and $\varphi^2$, respectively. 
Also, $\Theta_r[2][1] = 1$ because $\mathcal{L}(s^2_0)$ is on $\omega^1$. 
Therefore, robot ID $2$ comes before $1$ in $\Theta_a$. 
\end{example}

\smallskip
\subsubsection{Time parameterization}
\label{subsubsec:time_parameterization}

Time parameterization ensures \textit{inter-robot collision avoidance} through decoupled planning. 
Recall that the CP does not change the remaining paths $\Sigma_{rem}$ for the non-participants, implicitly meaning the non-participants have higher priorities than the participants. 
So, for each participant $r^i$ in order of $\Theta_a$, we incrementally compute its \textit{start time offset} $\Upsilon[i] \in \mathbb{N}_0$ (line 7). 
It is the amount of time $r^i$ must wait at its start state to avoid collisions with its higher priority robots (some other participants and the non-participants). 
But, a collision between a participant $r^i$ and non-participant $r^j$ may become inevitable (see Example \ref{ex:sto}). 
In that case, the CP inactivates the participant in $\Gamma$ and reexamines updated paths $\Phi$ (line 11) for the presence of new crossover paths. 
In the worst case, all the participants may get inactivated due to inevitable collisions with the non-participants. 
Otherwise, $r^i$'s collision-free path $\sigma'^i \in \Sigma'$ gets generated (line 12) by inserting its start state $s^i_0$ at the beginning of $\omega^i$ for $\Upsilon[i]$ times. 

\begin{example}
As a continuation of Example \ref{ex:prioritization}, start time offsets for the participants $r^2$ is $\Upsilon[2] = 0$, and $r^1$ is $\Upsilon[1] = 2$ because any value smaller than $2$ would cause a collision. 
Therefore, the collision-free path lengths for the participants $r^1$ and $r^2$ are $|\sigma'^1| = 4$ and $|\sigma'^2| = 3$, respectively. 
Observe that the horizon length $\lambda = 3$. 
\end{example}


\begin{figure}
    \centering
    \includegraphics[scale=0.25]{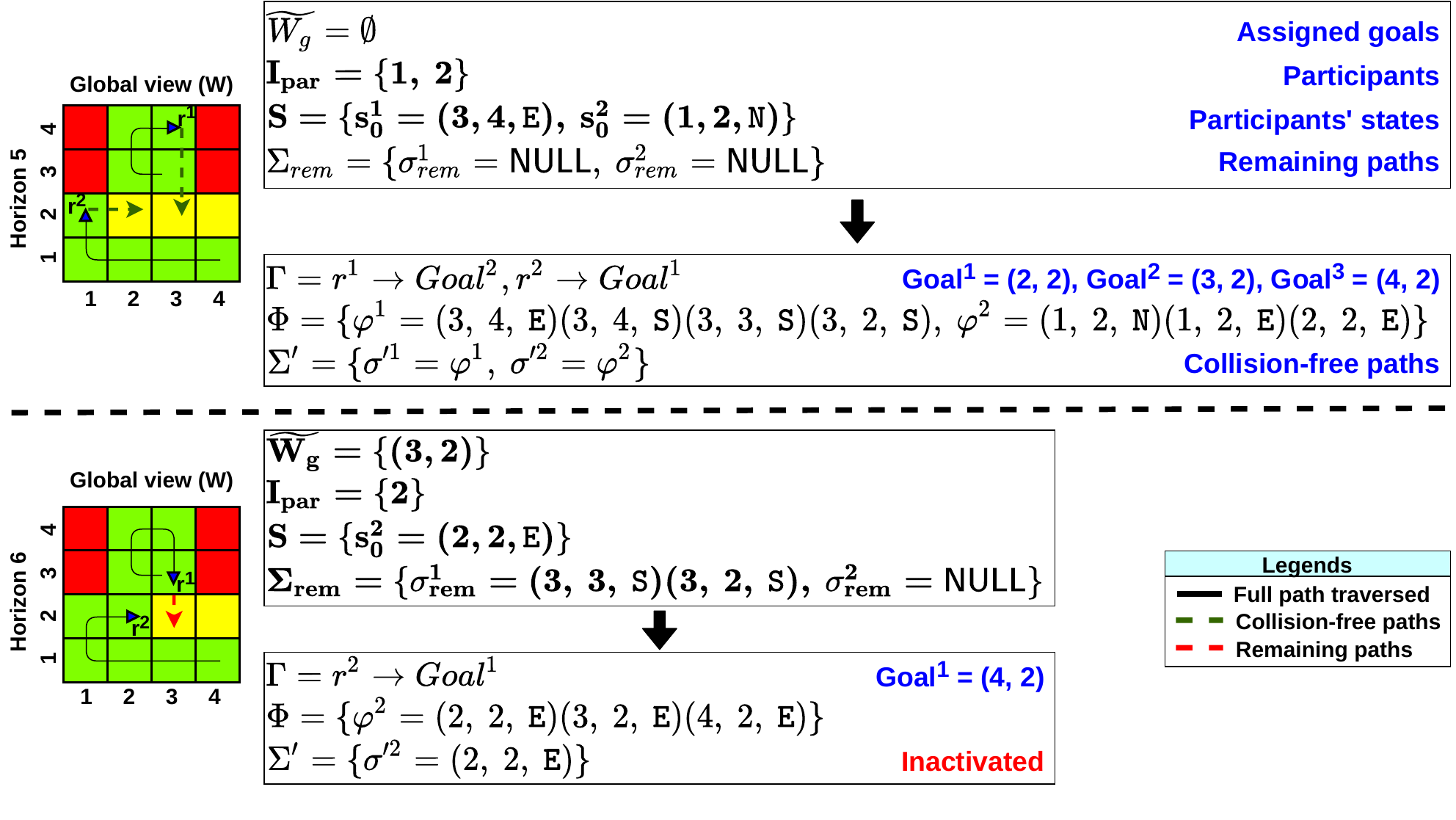}
    \caption{Inactivation of a participant during offsetting}
    \label{fig:sto}
\end{figure}

\begin{example}
\label{ex:sto}
In Figure \ref{fig:sto}, we show an extreme case of replanning where the CP inactivates a participant to avoid collision with a non-participant. 
Notice that both the robots $r^1$ and $r^2$ have traversed their full paths of length five till horizon $4$. 
In horizon $5$, both participate because none has any remaining path. 
The CP assigns the unassigned goal ids $2$ and $1$ to $r^1$ and $r^2$, respectively, and subsequently finds their collision-free paths $\Sigma'$. 
As the horizon length ($\lambda$) is found to be $2$, the CP reserves the goal $(3,\ 2)$ for $r^1$ in $\widetilde{W_g}$ and sets its remaining path $\sigma^1_{rem}$ accordingly. 
Therefore, only $r^2$ participates in horizon $6$. 
The CP finds $r^2$'s optimal path $\varphi^2$, leading to the only unassigned goal $(4,\ 2)$. 
Though the prioritization succeeds trivially, the offsetting fails because the non-participant $r^1$ would reach its goal $(3,\ 2)$ before the participant $r^2$ gets past that goal. 
So, the CP inactivates $r^2$. 
In this horizon, only $r^1$ remains active with $\lambda = 1$. 
Observe that both will participate in the next horizon. 
\end{example}

%% file: 6_theoretical_analysis.tex
\section{Theoretical Analysis}
\label{sec:theoretical_analysis}

First, we formally prove that \FnOnDemCPP covers $W$ completely and then analyze its time complexity. 

\input{6a_proof.tex}
\input{6b_time_complexity.tex}

%% file: 6a_proof.tex
\longversion
{
\subsection{Proof of Complete Coverage}
\label{subsec:proof}

\begin{lemma}[Lemma $4.3$ of \cite{DBLP:conf/iros/MitraS22}]
\label{lemma:gamrcpp_horizon}
\FnGAMRCPPHorizon ensures that at least one goal gets visited in each horizon. 
\end{lemma}

\begin{lemma}
\label{lemma:asyncpp_gamrcpp_equivalence}
\FnCPPForPar $\equiv$ \FnGAMRCPPHorizon if \mbox{$R^* = R$}. 
\end{lemma}

\begin{proof}
When all the robots are participants, i.e., $R^* = R$, there are no non-participants, i.e., $\overline{I_{par}} = \emptyset$. 
It implies that in $\Sigma_{rem}$, $\sigma^i_{rem} = \mathsf{NULL}, \forall i \in [R]$, thereby $\widetilde{W_g} = \emptyset$. 
Hence, \FnCPPForPar $\equiv$ \FnGAMRCPPHorizon if $R^* = R$. 
\end{proof}

\begin{lemma}
\label{lemma:asyncpp_horizon}
\FnOnDemCPPHor ensures that at least one robot visits its goal in each horizon. 
\end{lemma}

\begin{proof}
The set of all collision-free paths $\Sigma'_{all}$ (line 47 of Algorithm \ref{algo:ondemcpp}) consists of two kinds of paths, namely the paths $\Sigma'$ of the participants $I_{par}$ and the remaining paths $\Sigma_{rem}$ of the non-participants $\overline{I_{par}}$. 
Now, consider two cases: (\romannum{1}) $R^* = R$ and (\romannum{2}) $R^* < R$. 
In the case of (\romannum{1}), $\widetilde{W_g} = \emptyset$. 
Further, by Lemma~\ref{lemma:asyncpp_gamrcpp_equivalence} and Lemma~\ref{lemma:gamrcpp_horizon}, when $W_g \neq \emptyset$ (lines 39-40), there is at least one active robot in $\Sigma'$, and so in $\Sigma'_{all}$. 
In the case of (\romannum{2}), however, $\widetilde{W_g} \neq \emptyset$ because there are $R - R^*$ non-participants. 
Thus, $\Sigma'_{all}$ contains at least one active robot because the non-participants are active. 
Moreover, the horizon length $\lambda$ (line 48) ensures that at least one robot (participant or non-participant) remains active in the $\lambda$-length paths $\Sigma$ (line 49) to visit its goal. 
\end{proof}
}
\begin{theorem}
\label{theorem:asyncpp_complete_coverage}
\FnOnDemCPP eventually terminates, and when it does, it ensures complete coverage of $W$. 
\end{theorem}

\begin{proof}
\shortversion
{
Please refer to \cite{DBLP:conf/arxiv/MitraS23} for the proof.
}
\longversion
{
By Lemma \ref{lemma:asyncpp_horizon}, in a horizon, at least one active robot visits its goal, thereby exploring $W_u$ (as $W_{free}$ is strongly connected) and adding new goals (if any) into $W_g$. 
If an active participant visits its goal, the CP marks that goal as covered. 
So, in the next horizon, $W_g$ decreases, and $W_c$ increases. 
But, if a non-participant (inherently active) visits its goal, the CP removes that reserved goal from $\widetilde{W_g}$ in the next horizon, thereby decreasing $\widetilde{W_g}$. 
But, it may not increase $W_c$ as other robots might have already passed through that goal in the past horizons, already covering it in $W_c$. 
These two types of active robots and the inactive participants $I^{inact}_{par}$ participate in the next horizon (line 13 of Algorithm \ref{algo:ondemcpp}). 
However, the rest of the active robots who cannot visit their goals in the current horizon become non-participants in the next horizon as the CP reserves their goals into $\widetilde{W_g}$. 
Each reserved goal in $\widetilde{W_g}$ gets visited eventually in some horizon (as the CP does not alter the corresponding robot's remaining path), making the corresponding robot participant in the next horizon. 
Therefore, eventually, all the robots become participants (i.e., $I_{par} = [R]$), falsifying the \textbf{if} condition (line 41), and hence $\widetilde{W_g} = \emptyset$. 
Now, if $W_g = \emptyset$ (line 39), it entails $W_c = W_{free}$ (lines~45-46). 
}
\end{proof}

\longversion
{
Note that \FnOnDemCPP can also ensure complete coverage of a $W$, where $W_{free}$ is not strongly connected, provided we deploy at least one robot in each component. 
}

%% file: 6b_time_complexity.tex
\longversion
{
\subsection{Time Complexity Analysis}
\label{subsec:time_complexity}

%
%

\begin{lemma}
\label{lemma:cfp_for_par_tc}
\FnCFPForPar takes $\mathcal{O}(R^3)$. 
\end{lemma}

\begin{proof}
The only difference between \FnCFPForPar (Algorithm \ref{algo:cfp_for_par}) and $\mathtt{GAMRCPP\_CFP}$ (Algorithm 3 in \cite{DBLP:conf/iros/MitraS22}) is that $\mathtt{compute\_sto}$ (line 7 of Algorithm \ref{algo:cfp_for_par}) gets called just after the \textbf{while} loop in $\mathtt{GAMRCPP\_CFP}$. 
Still, the time complexity of $\mathtt{compute\_sto}$ and the rest of the body of the \textbf{while} loop (lines 1-11 of Algorithm \ref{algo:cfp_for_par}) remain $\mathcal{O}(R^3)$ (interested reader may read Theorem 4.6 of \cite{DBLP:conf/iros/MitraS22}) as $R^* = \mathcal{O}(R)$. 
\end{proof}
\begin{lemma}
\label{lemma:cpp_for_par_tc}
\FnCPPForPar takes $\mathcal{O}(|W|^3)$. 
\end{lemma}

\begin{proof}
\FnCPPForPar (Algorithm \ref{algo:cpp_for_participants}) does the task performed by lines 1-4 of \FnGAMRCPPHorizon (Algorithm 2 in \cite{DBLP:conf/iros/MitraS22}). 
Therefore, \FnCPPForPar takes as much time as \FnGAMRCPPHorizon, which is $\mathcal{O}(|W|^3)$ (interested reader may read Theorem 4.6 of \cite{DBLP:conf/iros/MitraS22}). 
\end{proof}
\begin{lemma}
\label{theorem:asyncpp_tc}
\FnOnDemCPPHor takes $\mathcal{O}(|W|^3)$. 
\end{lemma}

\begin{proof}
Finding the unassigned goals $W_g \setminus \widetilde{W_g}$ (line 39 of Algorithm \ref{algo:ondemcpp}) takes $\mathcal{O}(|W|)$ as both $W_g$ and $\widetilde{W_g}$ are bounded by $\mathcal{O}(|W|)$. 
As per Lemma ~\ref{lemma:cpp_for_par_tc}, \FnCPPForPar (line 40) takes $\mathcal{O}(|W|^3)$. 
Checking the existence of non-participants (line 41) takes $\mathcal{O}(R)$. 
The generation of the zero-length paths for the participants (lines 42-44) takes $\mathcal{O}(R^*)$. 
Combining all the paths into $\Sigma'_{all}$ (line 47) and determining the horizon length $\lambda$ (line 48) take $\mathcal{O}(R)$. 
Subsequently, generating the $\lambda$-length paths $\Sigma$ (lines 49 and 51-70) take $\mathcal{O}(R)$ as the \textbf{for} loop in \FGetEqLenPaths (lines 54-69) takes $\mathcal{O}(R)$. 
So, \FnOnDemCPPHor takes $\mathcal{O}((|W| + |W|^3) + (R + R^*) + 3 \cdot R)$, which can be written as $\mathcal{O}(|W|^3)$ as $R = \mathcal{O}(|W|)$. 
\end{proof}
}


\begin{theorem}
\label{theorem:asyncpp_total_tc}
\FnOnDemCPP's time complexity is $\mathcal{O}(|W|^4)$. 
\end{theorem}

\begin{proof}
\shortversion
{
Please refer to \cite{DBLP:conf/arxiv/MitraS23} for the proof.
}
\longversion
{
Each call to \FnUpdateGlobalview (line 10 of Algorithm \ref{algo:ondemcpp}) takes $\mathcal{O}(|W|^2)$ as follows: three \textbf{for} loops (lines 21-23, 24-26, and 28-30) and computing $W_g,\ W_o$, and $W_u$ (lines 27, 31-32) take $\mathcal{O}(|W|^2)$. 
In a horizon, \mbox{$N_{req} = \mathcal{O}(R)$} (line 11). 
So, \FnUpdateGlobalview total takes $\mathcal{O}(R \cdot |W|^2)$, which is $\mathcal{O}(|W|^3)$. 
According to Lemma ~\ref{lemma:asyncpp_horizon}, at least one robot visits its goal in each horizon. 
In the worst case, in each horizon, only one robot visits its goal while the rest remain inactive or make progress toward their goals. 
As the number of horizons is at most $|W_g|$, the service \FRcvLV invokes \FnOnDemCPPHor (line 15) at most $|W_g|$ times. 
So, \FnOnDemCPP takes $\mathcal{O}(|W_g| \cdot (|W|^3 + |W|^3))$, which is $\mathcal{O}(|W|^4)$. 
}
\end{proof}


%% file: 7_evaluation.tex
\section{Evaluation}
\label{sec:evaluation}

\shortversion
{
    \input{tables/results_table_v5_short.tex}
}
\longversion
{
    \input{tables/results_table_v5.tex}
}

\subsection{Implementation and Experimental Setup}
\label{subsec:implementation_and_experimental_setup}

We implement \FnRobot (Algorithm \ref{algo:robot}) and \FnOnDemCPP (Algorithm \ref{algo:ondemcpp}) in a common ROS \longversion{\cite{key_ros}} package\footnote{https://github.com/iitkcpslab/OnDemCPP} 
and run it in a computer having Intel\textsuperscript{\textregistered} Core\textsuperscript{\texttrademark} i7-4770 CPU @ 3.40 GHz and 16 GB of RAM. 
For experimentation, we consider eight large $2$D grid benchmark workspaces from \cite{DBLP:conf/socs/SternSFK0WLA0KB19} of varying size and obstacle density. 
In each workspace, we \textit{incrementally} deploy $R \in \{128, 256, 512\}$ robots and repeat each experiment $10$ times with different initial deployments of the robots to report their mean in the performance metrics (standard deviations are available in~\cite{DBLP:conf/arxiv/MitraS23}).

\subsubsection{Robots for evaluation}

We consider both TurtleBot (introduced before) and Quadcopter in the experimentation. 
\input{4b_quadcopter_2d}

\subsubsection{Evaluation metrics}

We compare \FnOnDemCPP with \FnGAMRCPP\cite{DBLP:conf/iros/MitraS22} in terms of \textit{the mission time} ($T_m$), which is the sum of \textit{the total computation time} ($T_c$) and \textit{the total path execution time} ($T_p$) while ignoring \textit{the total communication time} between the CP and the robots. 
\textit{Total computation time} ($T_c$) is the sum of the computation times spent on \FnOnDemCPPHor across all horizons. 
\textit{Total path execution time} ($T_p$) can be expressed as $T_p = \Lambda \times \tau$, where $\Lambda$ is \textit{the total horizon length} (i.e., the sum of the horizon lengths). 
We can also express $T_p$ as the sum of $T_{Halt}$ and $T_{non-Halt}$, where $T_{Halt}$ is the duration for which the robots remain stationary while following respective paths, i.e., when they execute $\mathtt{Halt(H)}$ moves, and $T_{non-Halt}$ is the duration for which the robots move, i.e., when they do not. 

In our evaluation, we take the path cost as the number of moves the corresponding robot performs, where each move takes $\tau = 1\si{\second}$. 
This is realistic because the maximum translational velocity of a TurtleBot2 is $0.65\si{\meter}/\si{\second}$~\cite{key_turtlebot2}, and that of a Quadcopter is $\sim$ $16\si{\meter}/\si{\second}$~\cite{dji_drones}. 
Thus, to ensure $\tau = 1\si{\second}$, we can keep the grid cell size for a TurtleBot2 up to $0.65\si{\meter}$, which is sufficient considering its size. 
Similarly, a grid cell size of up to $16\si{\meter}$ is sufficient for most practical applications involving Quadcopters. 
We demonstrate this in the real experiments described in Section \ref{subsec:sim_and_real}.

\subsection{Results and Analysis}
\label{subsec:results_analysis}

We show the experimental results in Table \ref{tab:experimental_results}, where we list workspaces in increasing order of their obstacle-free cell count $|W_{free}|$ for each type of robot. 
In the table, $R^*$ denotes the average number of participants over all horizons. 

\longversion
{
    \input{7b_evaluation_ext.tex}
}

\subsubsection{Comparison of Mission time}

The total computation time $T_c$ increases with the number of robots $R$ because more robots become participants in a horizon. 
So, assigning $R^*$ participants to $G^*$ unassigned goals and getting their collision-free paths $\Sigma'$ becomes computationally intensive. 
Notice that $T_c$ also increases with the workspace size, specifically with $|W_{free}|$ as the \textit{state-space} increases. 
Unlike \FnGAMRCPP, where all $R$ robots participate in replanning, \FnOnDemCPP replans with $R^* \leq R$ robots. 
It results in substantially smaller $T_c$ in \FnOnDemCPP compared to \FnGAMRCPP.     

Next, the total horizon length $\Lambda$ decreases with $R$ as the deployment of more robots expedites the coverage. 
But, it increases with the workspace size since more $W_{free}$ must be covered. 
During replanning, \FnGAMRCPP finds paths for $R$ robots without constraint. 
But, \FnOnDemCPP replans for $R^*$ participants while respecting the non-participants' constraint $\Sigma_{rem}$. 
As a result, there is an increase in the participants' collision-free path lengths, increasing individual horizon length $\lambda$ and so $\Lambda$. 
Thus, \FnGAMRCPP yields better $T_p$ compared to \FnOnDemCPP. 
\shortversion
{
For more insights on $T_p$, please refer to Section \RomanNumeralCaps{5}-B.1 in \cite{DBLP:conf/arxiv/MitraS23}. 
}

In summary, \FnOnDemCPP outperforms \FnGAMRCPP in terms of $T_c$ but underperforms in terms of $T_p$. 
As $R$ or the workspace size increases, \FnOnDemCPP's gain in terms of $T_c$ surpasses its loss in terms of $T_p$ by a noticeable margin. 
So, \FnOnDemCPP beats \FnGAMRCPP w.r.t. the mission time $T_m$.

\subsubsection{Implication on Energy Consumption}

A ground robot like TurtleBot consumes energy only for $T_{non-Halt}$ duration. 
Table~\ref{tab:experimental_results} shows that $T_{non-Halt}$ is $7\%-57\%$ more for \FnOnDemCPP than for \FnGAMRCPP. 
Thus, the energy consumption for the ground robots for \FnOnDemCPP is also proportionately more than that for \FnGAMRCPP. 

The situation is quite different for aerial robots like quadcopters. 
Once the mission starts, a quadcopter either keeps on \textit{hovering} (during $T_c$ and $T_{Halt}$) or makes translational moves (during $T_{non-Halt}$). 
Thus, a quadcopter consumes energy throughout the mission, i.e., during $T_m$. 
As \FnOnDemCPP significantly reduces $T_m$ for hundreds of robots compared to \FnGAMRCPP, it also helps reduce power consumption in the quadcopters during a mission. 

Given that reducing energy consumption during a mission is more crucial for aerial robots than for ground robots and that \FnOnDemCPP significantly reduces $T_m$ for both types of robots, \FnOnDemCPP establishes itself to be superior to \FnGAMRCPP for hundreds of robots. 


\subsubsection{Limitation of OnDemCPP}

The results obtained from the \textit{maze-128-128-2} workspace for $R \in \{128,\ 256\}$ show the limitation of \FnOnDemCPP. 
In a \textit{cluttered} workspace with narrow passageways, the flexibility of \FnGAMRCPP allows it to find a more efficient $\Sigma'$, thereby outmatching \FnOnDemCPP in instances with a smaller $R$.

\subsection{Simulations and Real Experiments}
\label{subsec:sim_and_real}

\longversion{
\begin{figure*}[t]
    \centering
    \begin{subfigure}{0.49\linewidth}
        \centering
        \includegraphics[scale=0.28]{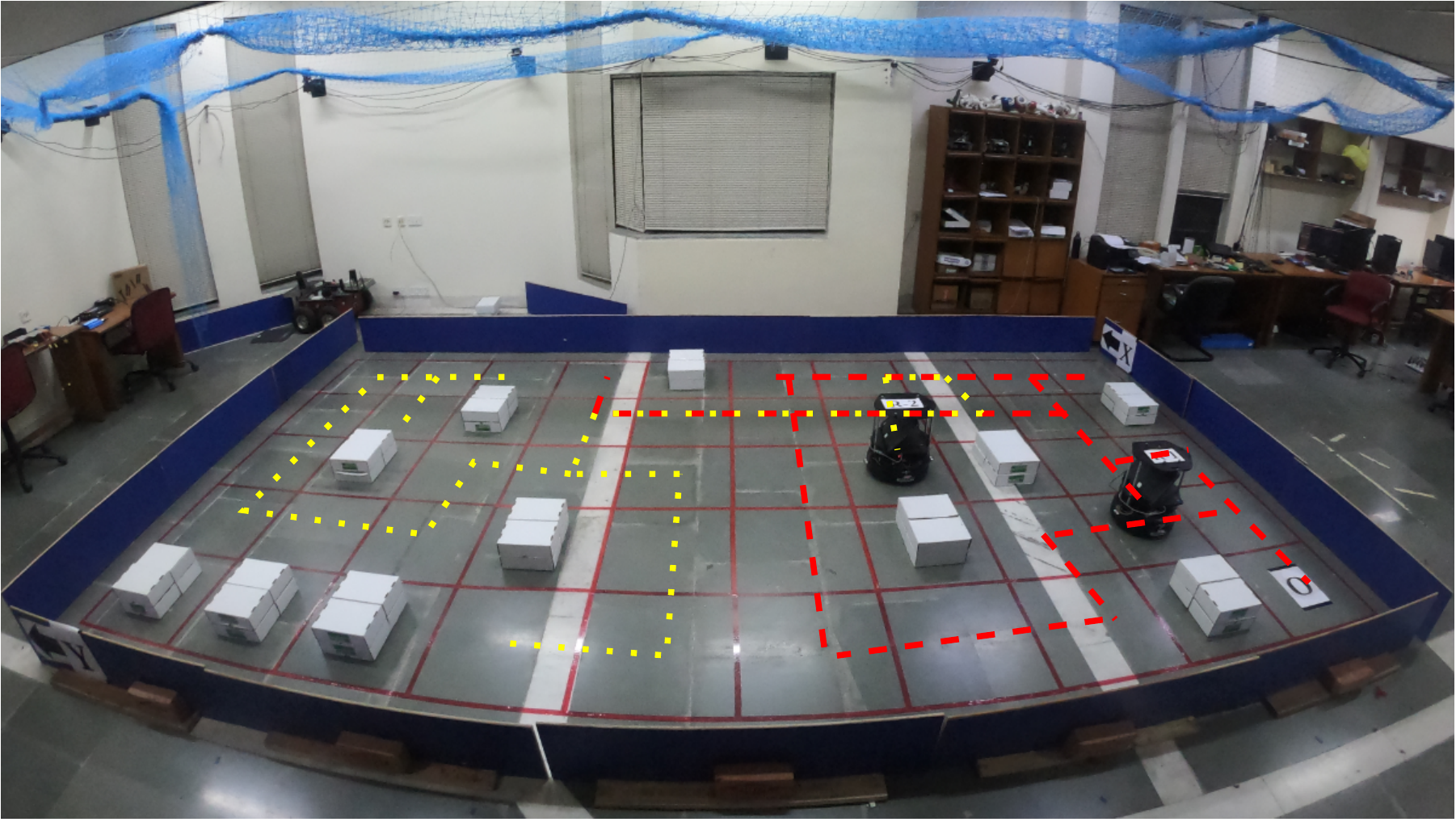}
        \caption{$\phantom{0}5\times10$ indoor workspace [cell size = $0.61\si{\meter}$]}       
        \label{subfig:indoor_ws}
    \end{subfigure}
    \hfill
    \begin{subfigure}{0.49\linewidth}
        \centering
        \includegraphics[scale=0.418]{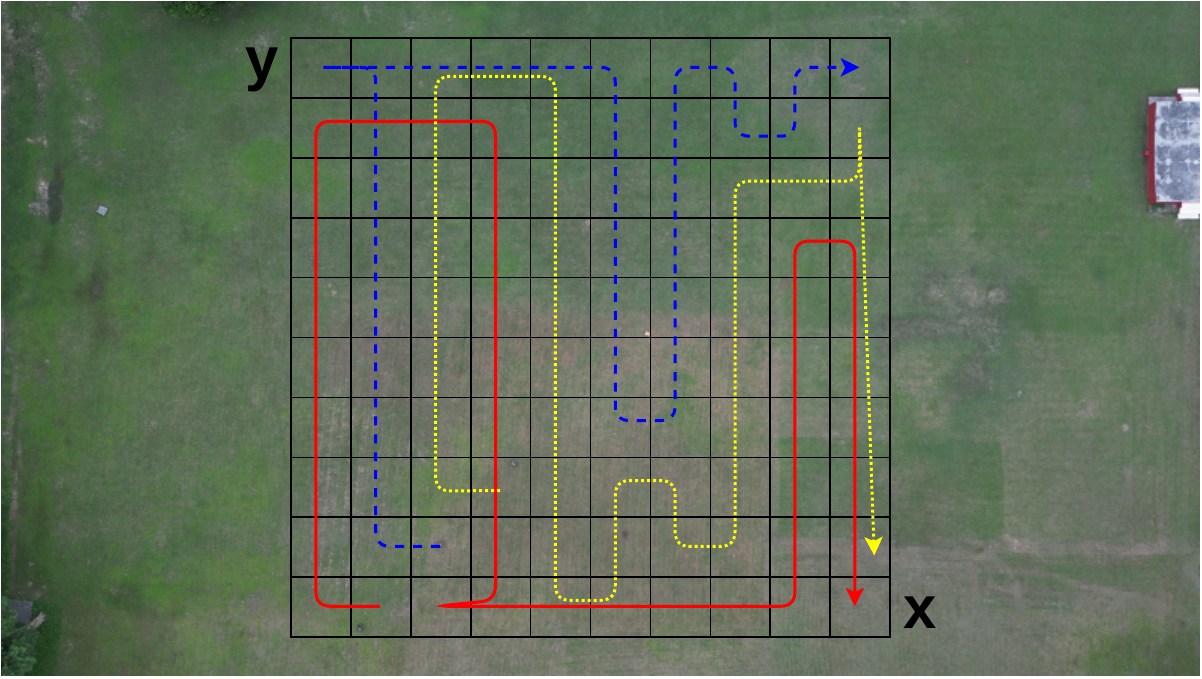}
        \caption{$10\times10$ outdoor workspace [cell size = $5\si{\meter}$]}       
        \label{subfig:outdoor_ws}
    \end{subfigure}
    \caption{Workspaces for the real experiments}
    \label{fig:real_ws}
\end{figure*}
}


For validation, we perform Gazebo\longversion{\cite{DBLP:conf/iros/KoenigH04}} simulations in five $2$D grid benchmark workspaces from \cite{DBLP:conf/socs/SternSFK0WLA0KB19} with $10$ Quadcopters and $10$ TurtleBots, respectively. 
We also perform two real experiments - one indoor with two TurtleBot2s, each fitted with four HC SR$04$ Ultrasonic Sound Sensors for obstacle detection and using Vicon\cite{key_vicon} for localization, 
and one outdoor with three Quadcopters, each fitted with one Cube Orange autopilot, one Herelink Air Unit for communication with the remote controller, and one Here$3$ GPS for localization. 
Figure \ref{fig:real_ws} shows the workspaces for these real experiments. 
The video containing the real experiments is available at
\url{https://youtu.be/5nhysTTp2Fw}.


\shortversion{
\begin{figure}[t]
    \centering
    \begin{subfigure}{0.49\linewidth}
        \centering
        \includegraphics[scale=0.025]{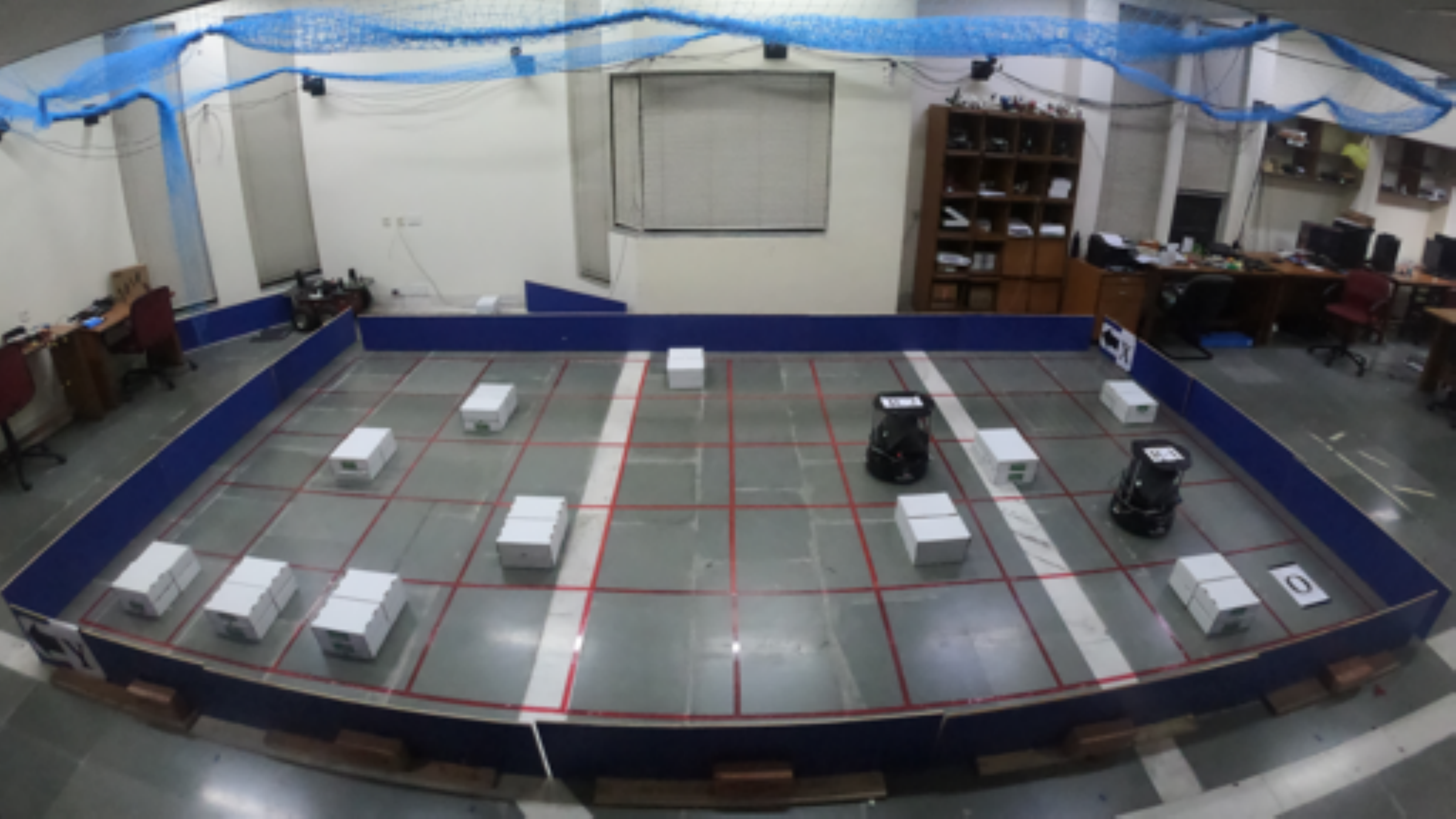}
        \caption{$\phantom{0}5\times10$ indoor workspace [cell size = $0.61\si{\meter}$]}       
        \label{subfig:indoor_ws}
    \end{subfigure}
    \hfill
    \begin{subfigure}{0.49\linewidth}
        \centering
        \includegraphics[scale=0.125]{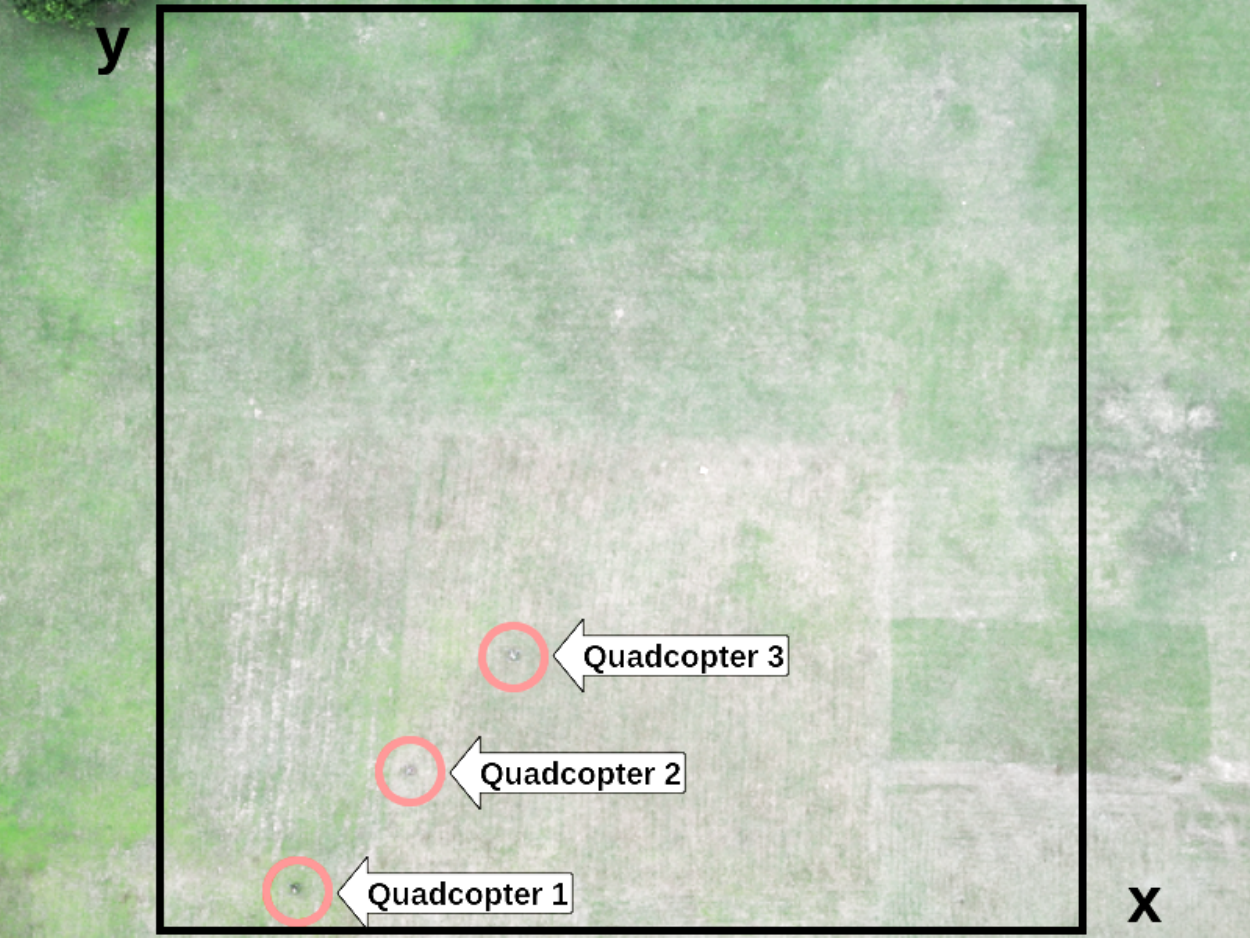}
        \caption{$10\times10$ outdoor workspace [cell size = $5\si{\meter}$]}       
        \label{subfig:outdoor_ws}
    \end{subfigure}
    \caption{Workspaces for the real experiments}
    \label{fig:real_ws}
\end{figure}
}


%% file: tables/results_table_v5_short.tex
\begin{table*}[t]
    \caption{Experimental results}
    \label{tab:experimental_results}
    \centering
    \small
    \resizebox{0.99\textwidth}{!}{      
        \begin{tabular}{@{}lccccrrrrrrrrrrc@{}}
            \toprule
            
             &  &  &  &  & \multicolumn{2}{c}{$T_c\ (\si{\second})$} & \multicolumn{2}{c}{$T_p\ (\si{\second})$} & \multicolumn{2}{c}{$T_{Halt}\ (\si{\second})$} & \multicolumn{2}{c}{$T_{non-Halt}\ (\si{\second})$} & \multicolumn{2}{c}{$T_m\ (\si{\second})$} & \textbf{Mission} \\
             
            \cmidrule(lr){6-7}
            \cmidrule(lr){8-9}
            \cmidrule(lr){10-11}
            \cmidrule(lr){12-13}
            \cmidrule(lr){14-15}
            
            $M$ & \multicolumn{2}{c}{\textit{Workspace}} & $R$ & $R^*$ & \FnGAMRCPP & \FnOnDemCPP & \FnGAMRCPP & \FnOnDemCPP & \FnGAMRCPP & \FnOnDemCPP & \FnGAMRCPP & \FnOnDemCPP & \FnGAMRCPP & \FnOnDemCPP & \textbf{speed up} \\
            
            \midrule
            
            \multirow{12}{*}{\begin{turn}{90}Quadcopter\end{turn}} & \multirow{3}{*}{\shortstack{w\_woundedcoast\\  $578 \times 642$\\ (34,020)}} & \multirow{3}{*}{\includegraphics[scale=0.03]{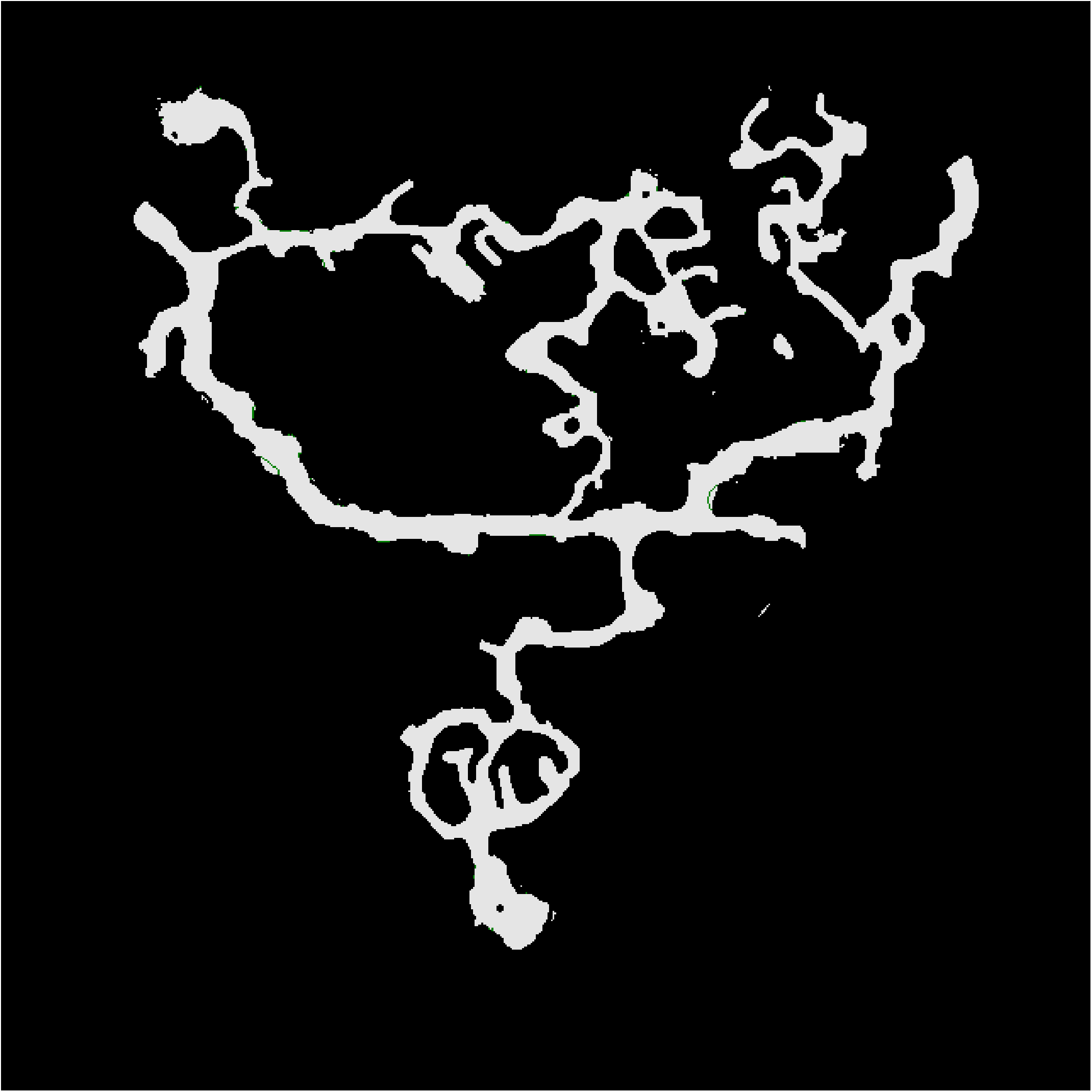}} & 128 & \phantom{0}80.4 & 619.5 & 362.2 & 696.2 & 1042.0 & 230.8 & 487.4 & 465.4 & 554.5 & 1315.7 & 1404.2 & 0.9 \\
             &  &  & 256 & 170.7 & 1091.3 & 521.7 & 416.8 & 804.2 & 182.3 & 512.2 & 234.5 & 291.9 & 1508.1 & 1325.9 & \textbf{1.1} \\
             &  &  & 512 & 361.5 & 1815.2 & 569.1 & 231.8 & 557.2 & 115.7 & 404.1 & 116.1 & 153.0 & 2047.0 & 1126.3 & \textbf{1.8} \\
             
            \cmidrule(lr){2-16}            
            
             & \multirow{3}{*}{\shortstack{Paris\_1\_256\\ $256 \times 256$\\ (47,240)}} & \multirow{3}{*}{\includegraphics[scale=0.03]{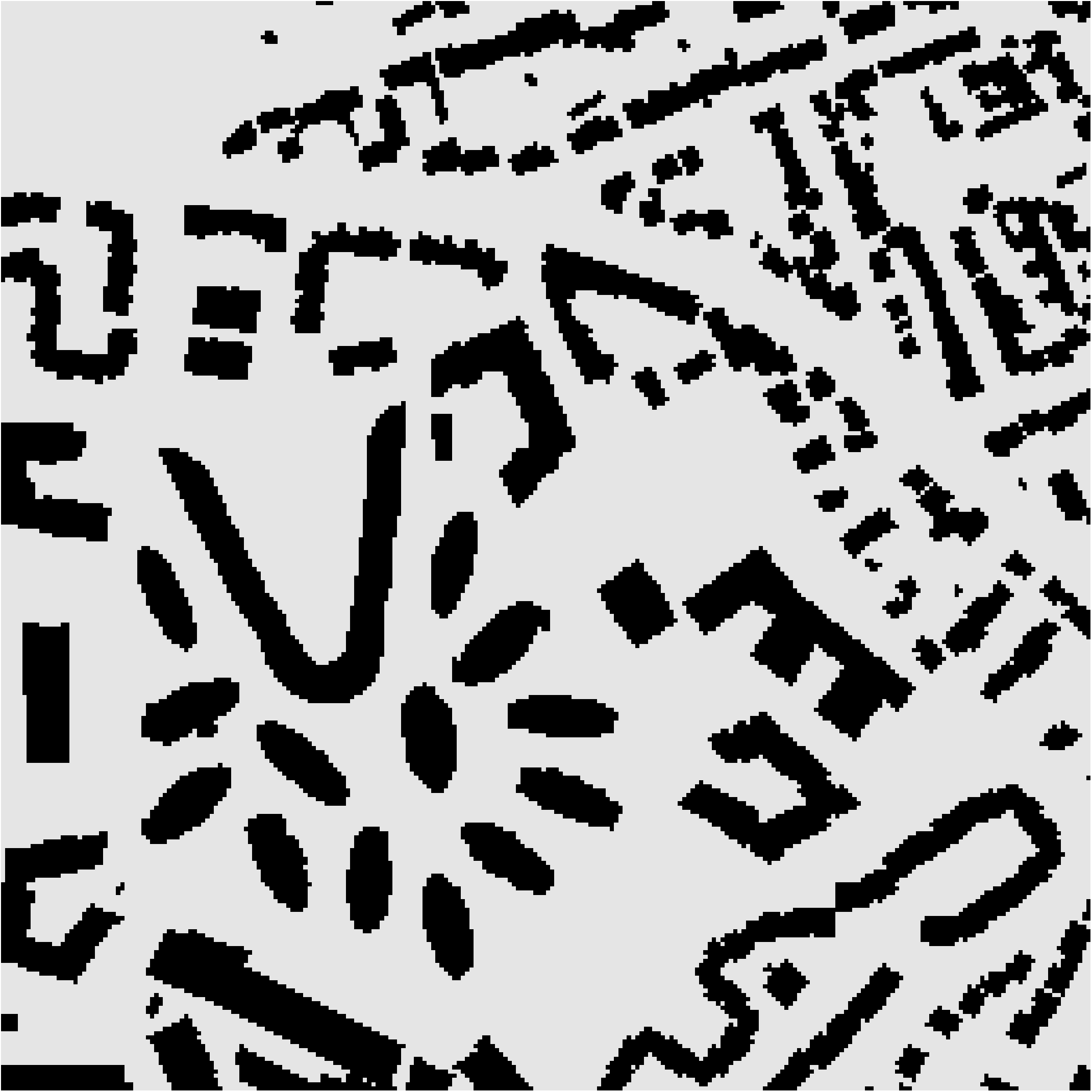}} & 128 & \phantom{0}81.3 & 513.5 & 270.4 & 815.2 & 962.6 & 214.9 & 298.2 & 600.3 & 664.3 & 1328.7 & 1233.0 & \textbf{1.1} \\
             &  &  & 256 & 157.2 & 1461.9 & 537.2 & 508.0 & 656.2 & 184.2 & 273.0 & 323.8 & 383.1 & 1969.9 & 1193.4 & \textbf{1.7} \\
             &  &  & 512 & 335.6 & 3338.8 & 894.0 & 301.1 & 455.7 & 137.6 & 247.1 & 163.5 & 208.5 & 3639.9 & 1349.7 & \textbf{2.7} \\
             
            \cmidrule(lr){2-16}            
            
             & \multirow{3}{*}{\shortstack{Berlin\_1\_256\\ $256 \times 256$\\ (47,540)}} & \multirow{3}{*}{\includegraphics[scale=0.03]{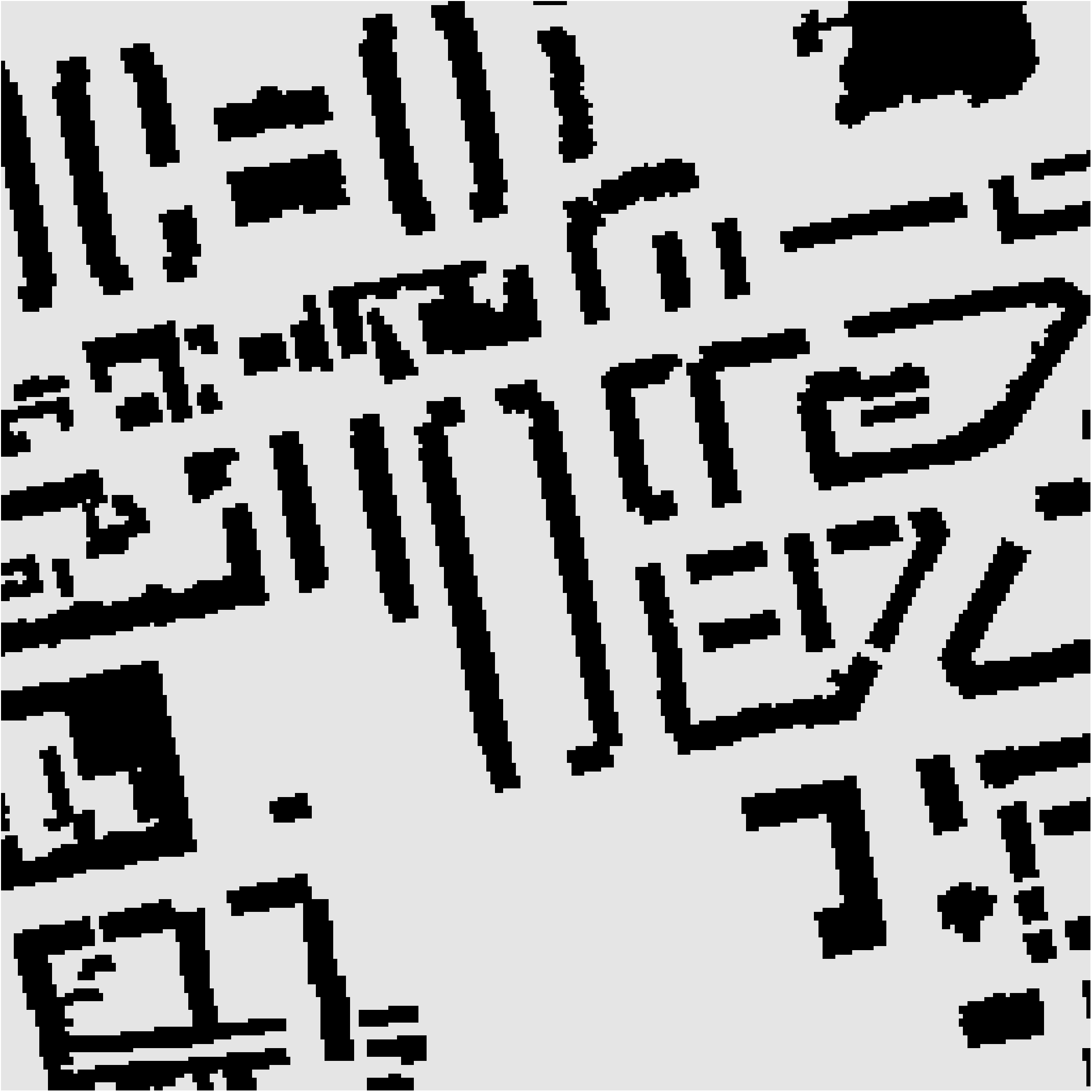}} & 128 & \phantom{0}83.2 & 489.3 & 238.4 & 752.7 & 898.2 & 169.9 & 272.9 & 582.8 & 625.2 & 1242.0 & 1136.6 & \textbf{1.1} \\
             &  &  & 256 & 167.4 & 1323.9 & 515.2 & 564.6 & 669.6 & 258.0 & 301.0 & 306.6 & 368.5 & 1888.5 & 1184.8 & \textbf{1.6} \\
             &  &  & 512 & 365.2 & 3196.5 & 856.0 & 434.5 & 524.9 & 275.8 & 324.8 & 158.7 & 200.0 & 3631.0 & 1380.9 & \textbf{2.6} \\
             
            \cmidrule(lr){2-16}            
            
             & \multirow{3}{*}{\shortstack{Boston\_0\_256\\ $256 \times 256$\\ (47,768)}} & \multirow{3}{*}{\includegraphics[scale=0.03]{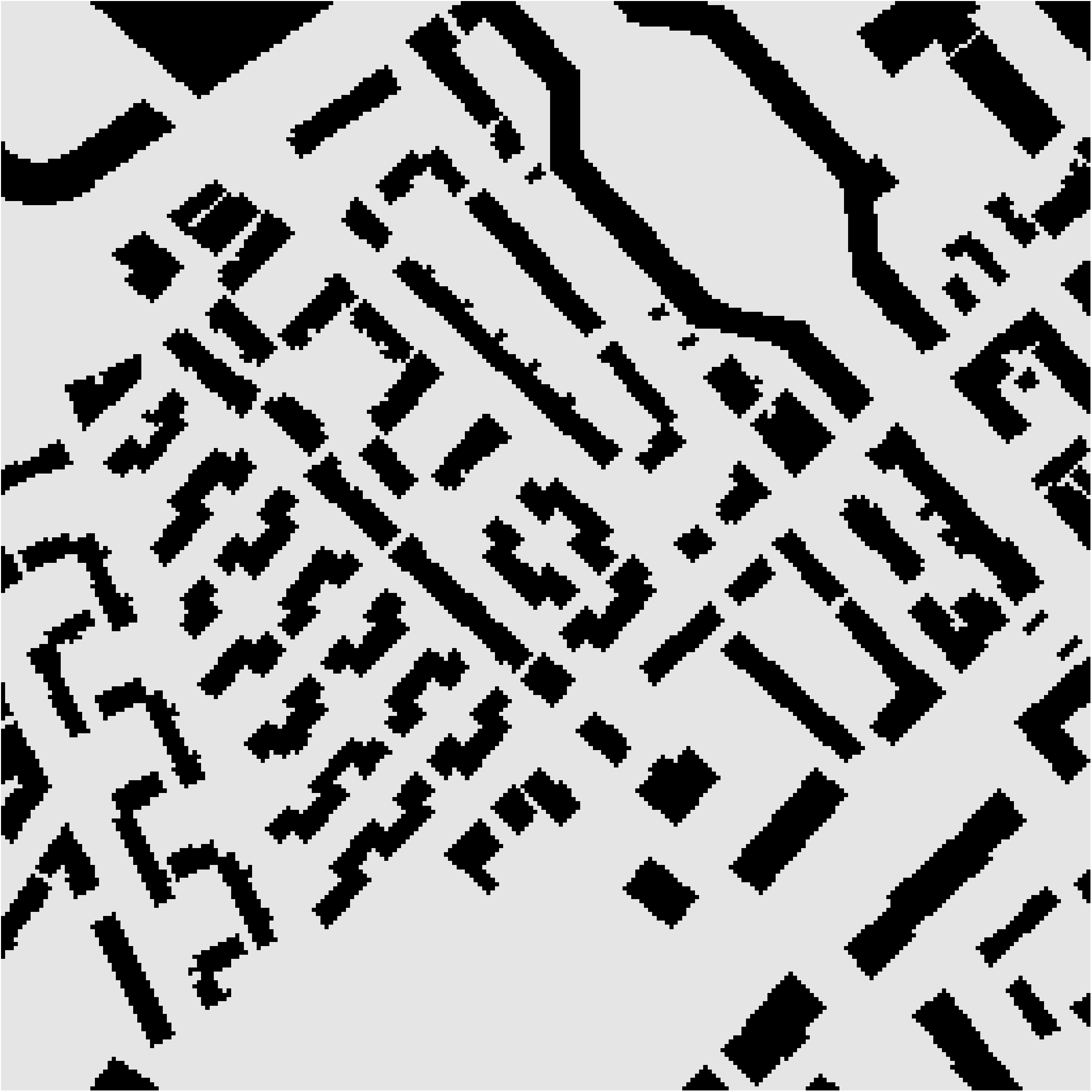}} & 128 & \phantom{0}79.6 & 509.5 & 250.6 & 773.6 & 928.0 & 169.4 & 262.6 & 604.2 & 665.3 & 1283.1 & 1178.6 & \textbf{1.1} \\
             &  &  & 256 & 157.8 & 1220.1 & 455.7 & 460.3 & 625.9 & 142.5 & 252.0 & 317.8 & 373.8 & 1680.4 & 1081.6 & \textbf{1.6} \\
             &  &  & 512 & 345.3 & 2454.7 & 628.7 & 252.9 & 428.7 & 95.6 & 236.7 & 157.3 & 191.9 & 2707.6 & 1057.4 & \textbf{2.6} \\
             
            \midrule
            
            \multirow{12}{*}{\begin{turn}{90}TurtleBot\end{turn}} & \multirow{3}{*}{\shortstack{maze-128-128-2\\ $128 \times 128$\\ (10,858)}} & \multirow{3}{*}{\includegraphics[scale=0.03]{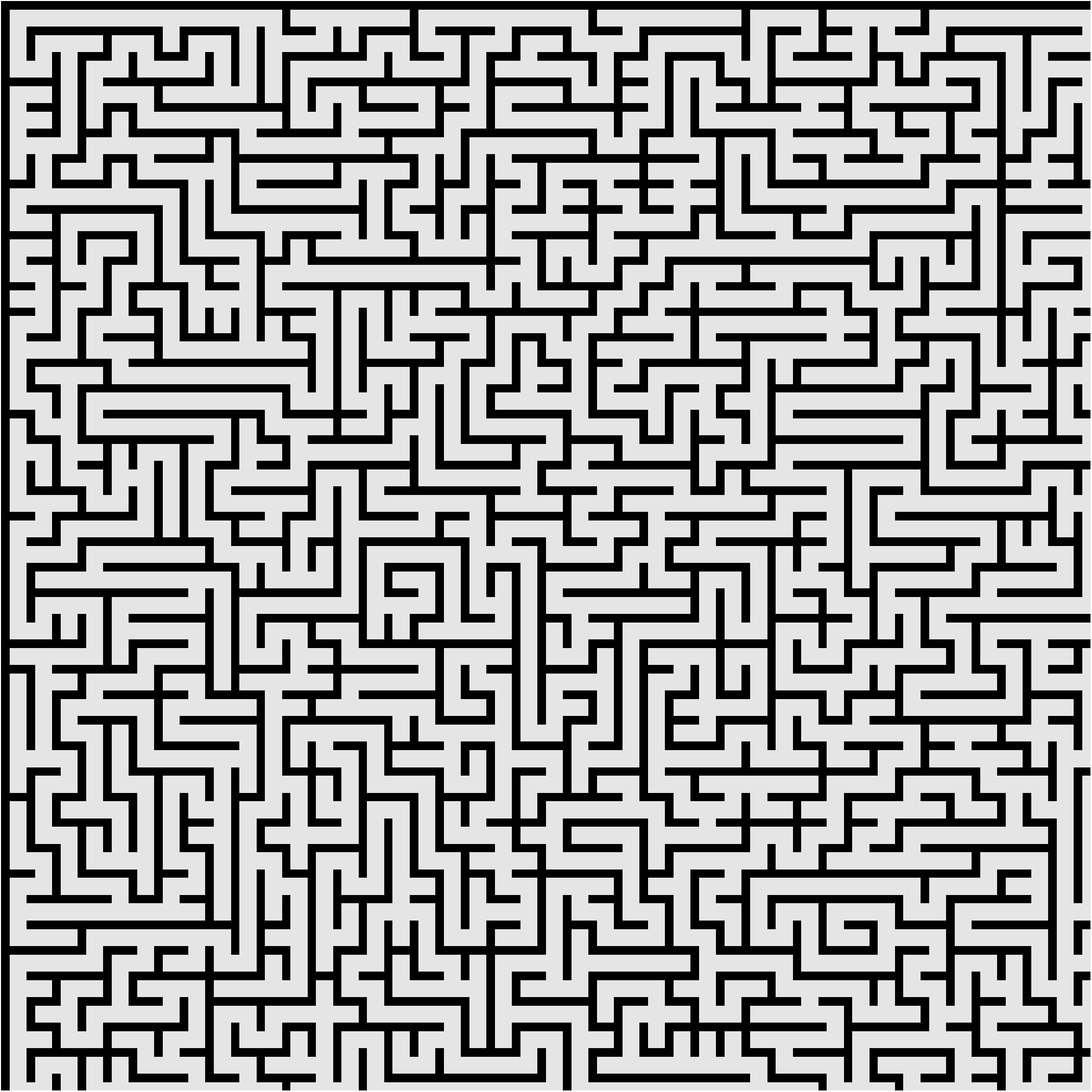}} & 128 & \phantom{0}75.5 & 90.8 & 49.5 & 451.9 & 791.5 & 217.1 & 509.8 & 234.8 & 281.6 & 542.7 & 841.0 & 0.6 \\
             &  &  & 256 & 174.7 & 193.6 & 68.1 & 218.6 & 417.6 & 115.7 &	294.8 & 102.9 & 122.7 & 412.2 & 485.7 & 0.8 \\
             &  &  & 512 & 378.4 & 411.1 & 124.1 & 125.9 & 224.3 & 79.4 & 165.9 & 46.5 & 58.3 & 537.0 & 348.4 & \textbf{1.5} \\
             
            \cmidrule(lr){2-16}            
            
             & \multirow{3}{*}{\shortstack{den520d\\ $257 \times 256$\\ (28,178)}} & \multirow{3}{*}{\includegraphics[scale=0.03]{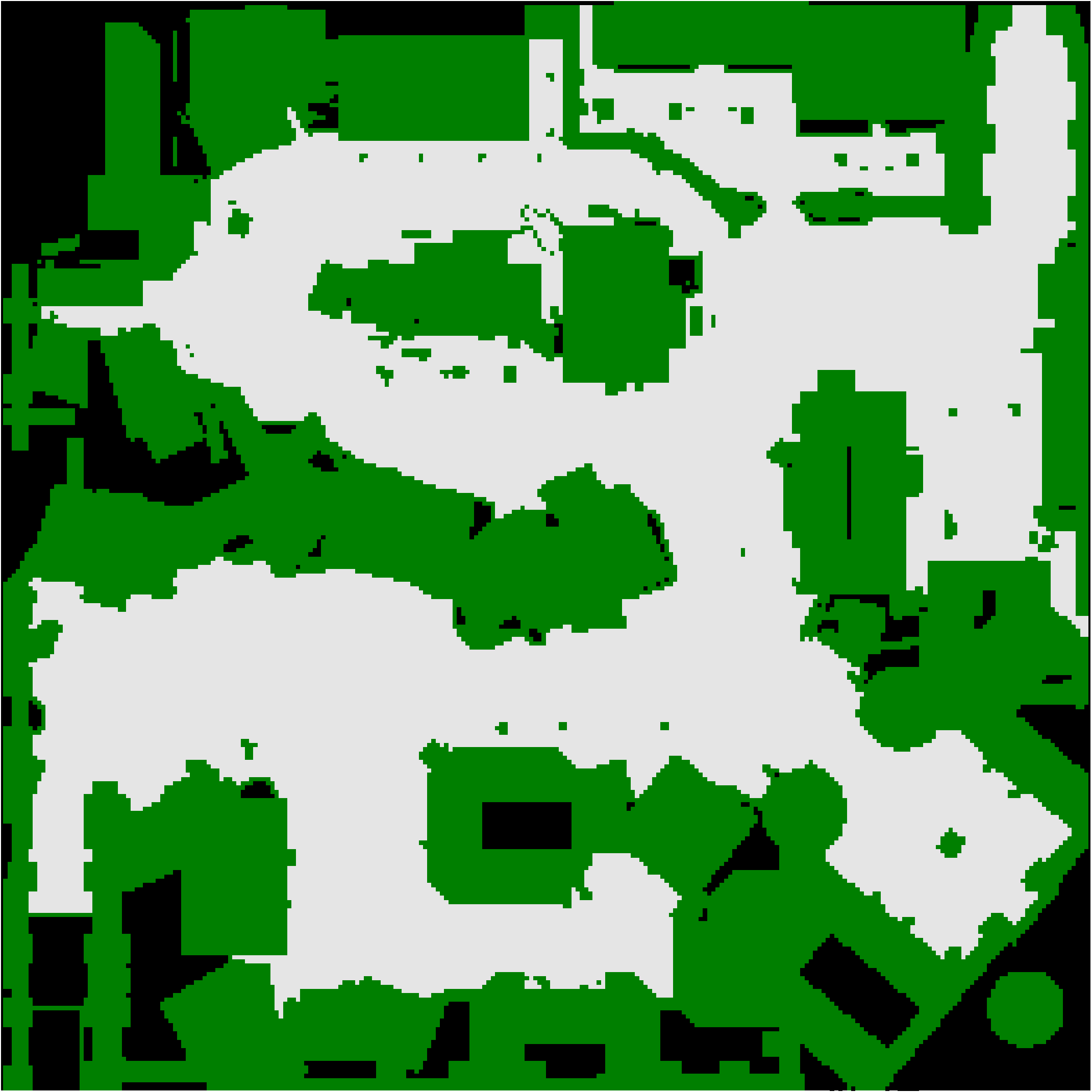}} & 128 & \phantom{0}71.9 & 341.8 & 147.6 & 556.4 & 773.1 & 151.6 & 286.1 & 404.8 & 486.9 & 898.2 & 920.7 & \textbf{1.0} \\
             &  &  & 256 & 153.8 & 685.4 & 253.9 & 324.2 & 506.4 & 118.5 & 247.1 & 205.7 & 259.3 & 1009.6 & 760.3 & \textbf{1.3} \\
             &  &  & 512 & 331.1 & 1292.4 & 316.0 & 188.7 & 398.6 & 82.8 & 252.9 & 105.9 & 145.6 & 1481.1 & 714.6 & \textbf{2.1} \\
             
            \cmidrule(lr){2-16}            
            
             & \multirow{3}{*}{\shortstack{warehouse-20-40-10-2-2\\ $164 \times 340$\\ (38,756)}} & \multirow{3}{*}{\includegraphics[scale=0.03]{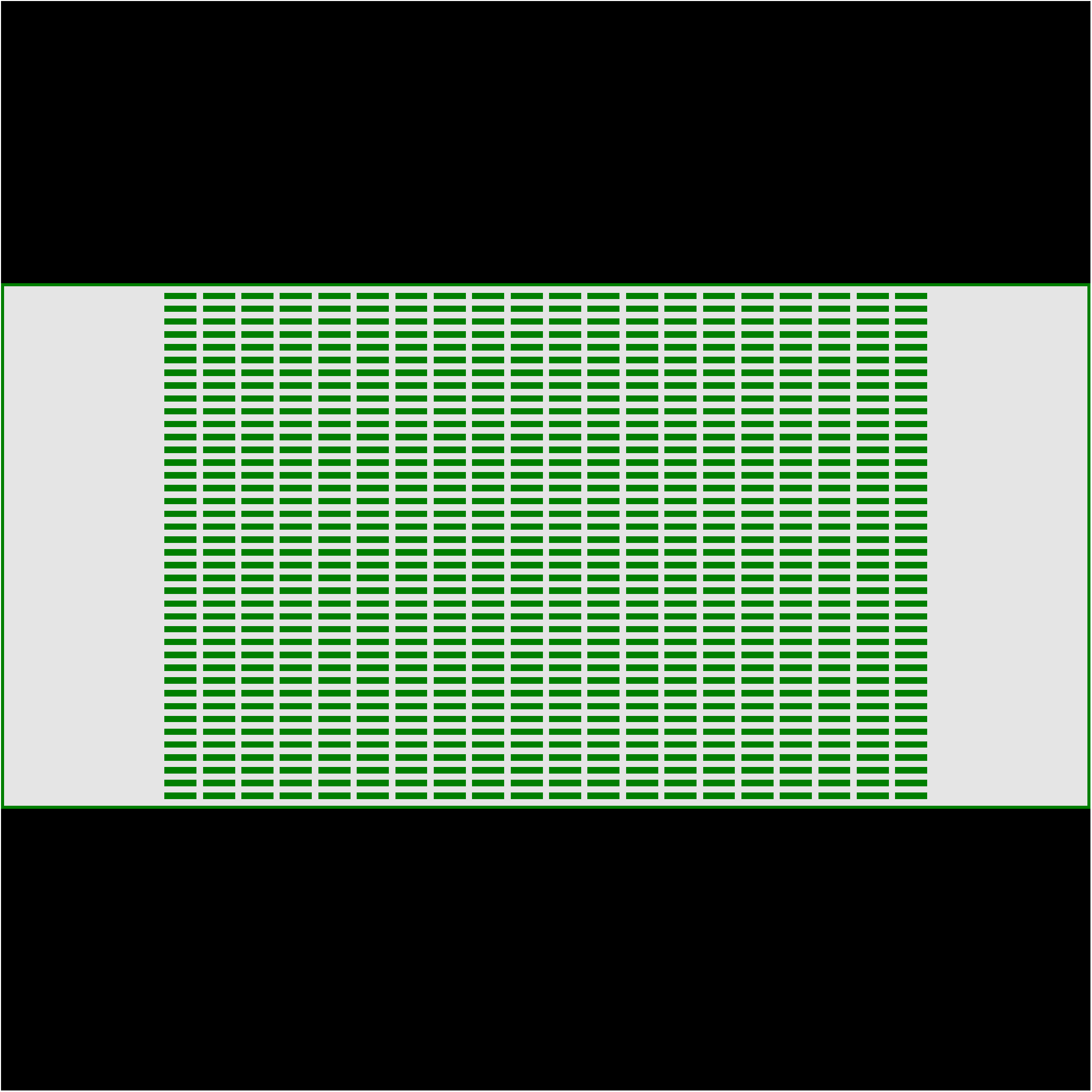}} & 128 & \phantom{0}81.2 & 467.1 & 211.3 & 595.6 & 707.6 & 115.2 & 187.2 & 480.4 & 520.3 & 1062.7 & 918.9 & \textbf{1.2} \\
             &  &  & 256 & 149.3 & 1009.5 & 336.0 & 361.6 & 502.5 & 102.2 & 182.8 & 259.4 & 319.6 & 1371.1 & 838.5 & \textbf{1.6} \\
             &  &  & 512 & 335.5 & 1646.2 & 394.7 & 208.9 & 375.2 & 75.5 & 204.8 & 133.4 & 170.3 & 1855.1 & 769.9 & \textbf{2.4} \\
             
            \cmidrule(lr){2-16}            
            
             & \multirow{3}{*}{\shortstack{brc202d\\ $481 \times 530$\\ (43,151)}} & \multirow{3}{*}{\includegraphics[scale=0.03]{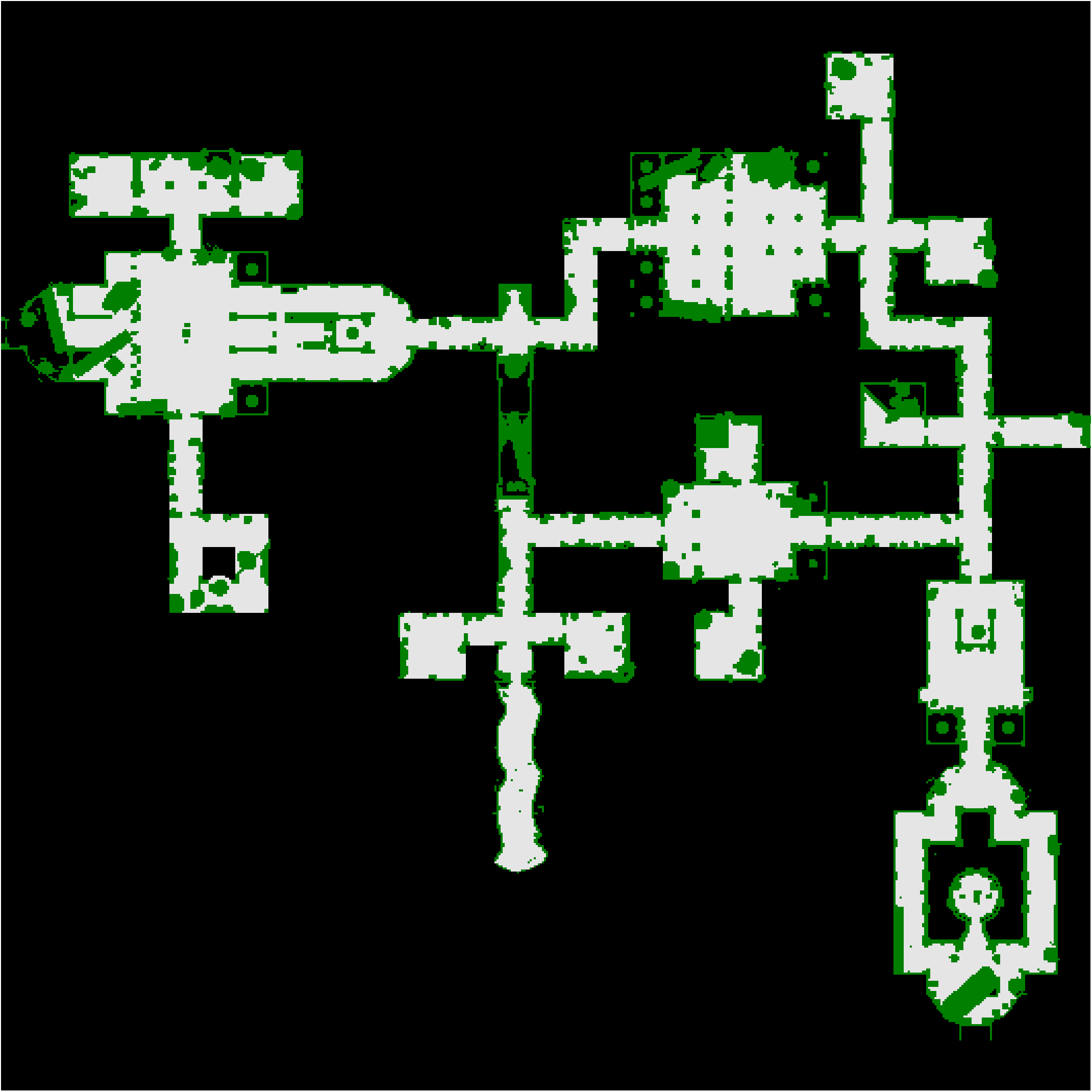}} & 128 & \phantom{0}65.1 & 1019.0 & 395.2 & 915.9 & 1444.9 & 242.5 & 587.4 & 673.4 & 857.4 & 1934.9 & 1840.1 & \textbf{1.1} \\
             &  &  & 256 & 142.0 & 1880.7 & 697.7 & 512.5 & 972.5 & 177.9 & 510.9 & 334.6 & 461.5 & 2393.2 & 1670.2 & \textbf{1.4} \\
             &  &  & 512 & 309.6 & 3026.8 & 705.6 & 302.6 & 780.9 & 133.9 & 514.6 & 168.7 & 266.2 & 3329.4 & 1486.5 & \textbf{2.2} \\
            \bottomrule
        \end{tabular}
   }
\vspace{-0.3cm}
\end{table*}

%% file: tables/results_table_v5.tex
\begin{table*}[t]
    \caption{Experimental results}
    \label{tab:experimental_results}
    \centering
    \small
    \resizebox{0.99\textwidth}{!}{
        \begin{tabular}{@{}lccccrrrrrrc@{}}
            \toprule
            
             &  &  &  &  & \multicolumn{2}{c}{$T_c\ (\si{\second})$} & \multicolumn{2}{c}{$T_p\ (\si{\second})$} & \multicolumn{2}{c}{$T_m\ (\si{\second})$} & \textbf{Mission} \\
             
            \cmidrule(lr){6-7}
            \cmidrule(lr){8-9}
            \cmidrule(lr){10-11}
            
            $M$ & \multicolumn{2}{c}{\textit{Workspace}} & $R$ & $R^*$ & \FnGAMRCPP & \FnOnDemCPP & \FnGAMRCPP & \FnOnDemCPP & \FnGAMRCPP & \FnOnDemCPP & \textbf{speed up} \\
            
            \midrule
            
            \multirow{12}{*}{\begin{turn}{90}Quadcopter\end{turn}} & \multirow{3}{*}{\shortstack{W1:\\ w\_woundedcoast\\ $578 \times 642$\ (34,020)}} & \multirow{3}{*}{\includegraphics[scale=0.029]{figures/workspaces/w_woundedcoast.pdf}} & 128 & \phantom{0}80.4 (62.9 \%) $\pm$ 3.8 & 619.5 $\pm$ \phantom{0}97.9 & 362.2 $\pm$ \phantom{0}58.4 & 696.2 $\pm$ 55.5 & 1042.0 $\pm$ 149.5 & 1315.7 $\pm$ 148.6 & 1404.2 $\pm$ 167.1 & 0.9 \\
             &  &  & 256 & 170.7 (66.7 \%) $\pm$ 5.7 & 1091.3 $\pm$ 255.3 & 521.7 $\pm$ 150.2 & 416.8 $\pm$ 95.0 & 804.2 $\pm$ 250.9 & 1508.1 $\pm$ 345.3 & 1325.9 $\pm$ 383.0 & \textbf{1.1} \\
             &  &  & 512 & 361.5 (70.6 \%) $\pm$ 9.7 & 1815.2 $\pm$ 359.4 & 569.1 $\pm$ 168.6 & 231.8 $\pm$ 33.0 & 557.2 $\pm$ 138.9 & 2047.0 $\pm$ 388.0 & 1126.3 $\pm$ 255.4 & \textbf{1.8} \\
             
            \cmidrule(lr){2-12}            
            
             & \multirow{3}{*}{\shortstack{W2:\\ Paris\_1\_256\\ $256 \times 256$\ (47,240)}} & \multirow{3}{*}{\includegraphics[scale=0.029]{figures/workspaces/Paris_1_256.pdf}} & 128 & \phantom{0}81.3 (63.6 \%) $\pm$ 2.6 & 513.5 $\pm$ \phantom{0}38.4 & 270.4 $\pm$ \phantom{0}37.8 & 815.2 $\pm$ 77.4 & 962.6 $\pm$ 100.8 & 1328.7 $\pm$ 109.8 & 1233.0 $\pm$ 118.8 & \textbf{1.1} \\
             &  &  & 256 & 157.2 (61.4 \%) $\pm$ 3.7 & 1461.9 $\pm$ 164.0 & 537.2 $\pm$ \phantom{0}79.6 & 508.0 $\pm$ 22.0 & 656.2 $\pm$ \phantom{0}49.4 & 1969.9 $\pm$ 172.2 & 1193.4 $\pm$ 110.8 & \textbf{1.7} \\
             &  &  & 512 & 335.6 (65.6 \%) $\pm$ 7.0 & 3338.8 $\pm$ 321.9 & 894.0 $\pm$ 113.6 & 301.1 $\pm$ 19.2 & 455.7 $\pm$ \phantom{0}50.1 & 3639.9 $\pm$ 333.6 & 1349.7 $\pm$ 117.1 & \textbf{2.7} \\
             
            \cmidrule(lr){2-12}            
            
             & \multirow{3}{*}{\shortstack{W3:\\ Berlin\_1\_256\\ $256 \times 256$\ (47,540)}} & \multirow{3}{*}{\includegraphics[scale=0.029]{figures/workspaces/Berlin_1_256.pdf}} & 128 & \phantom{0}83.2 (65.1 \%) $\pm$ 2.5 & 489.3 $\pm$ \phantom{0}36.5 & 238.4 $\pm$ \phantom{0}17.2 & 752.7 $\pm$ \phantom{0}32.6 & 898.2 $\pm$ \phantom{0}63.9 & 1242.0 $\pm$ \phantom{0}67.4 & 1136.6 $\pm$ \phantom{0}65.0 & \textbf{1.1} \\
             &  &  & 256 & 167.4 (65.4 \%) $\pm$ 12.1 & 1323.9 $\pm$ 154.2 & 515.2 $\pm$ \phantom{0}69.8 & 564.6 $\pm$ 132.5 & 669.6 $\pm$ \phantom{0}48.9 & 1888.5 $\pm$ 234.1 & 1184.8 $\pm$ 103.3 & \textbf{1.6} \\
             &  &  & 512 & 365.2 (71.3 \%) $\pm$ 34.5 & 3196.5 $\pm$ 329.4 & 856.0 $\pm$ 115.1 & 434.5 $\pm$ 209.9 & 524.9 $\pm$ 140.0 & 3631.0 $\pm$ 434.1 & 1380.9 $\pm$ 238.9 & \textbf{2.6} \\
             
            \cmidrule(lr){2-12}            
            
             & \multirow{3}{*}{\shortstack{W4:\\ Boston\_0\_256\\ $256 \times 256$\ (47,768)}} & \multirow{3}{*}{\includegraphics[scale=0.029]{figures/workspaces/Boston_0_256.pdf}} & 128 & \phantom{0}79.6 (62.2 \%) $\pm$ 4.5 & 509.5 $\pm$ \phantom{0}36.4 & 250.6 $\pm$ \phantom{0}30.1 & 773.6 $\pm$ 61.7 & 928.0 $\pm$ 100.6 & 1283.1 $\pm$ \phantom{0}94.0 & 1178.6 $\pm$ 116.3 & \textbf{1.1} \\
             &  &  & 256 & 157.8 (61.7 \%) $\pm$ 9.3 & 1220.1 $\pm$ 215.0 & 455.7 $\pm$ \phantom{0}61.2 & 460.3 $\pm$ 44.8 & 625.9 $\pm$ \phantom{0}46.9 & 1680.4 $\pm$ 256.8 & 1081.6 $\pm$ \phantom{0}92.4 & \textbf{1.6} \\
             &  &  & 512 & 345.3 (67.5 \%) $\pm$ 7.7 & 2454.7 $\pm$ 525.7 & 628.7 $\pm$ 119.3 & 252.9 $\pm$ 23.5 & 428.7 $\pm$ \phantom{0}69.2 & 2707.6 $\pm$ 546.1 & 1057.4 $\pm$ 182.2 & \textbf{2.6} \\
             
            \midrule
            
            \multirow{12}{*}{\begin{turn}{90}TurtleBot\end{turn}} & \multirow{3}{*}{\shortstack{W5:\\ maze-128-128-2\\ $128 \times 128$\ (10,858)}} & \multirow{3}{*}{\includegraphics[scale=0.029]{figures/workspaces/maze-128-128-2.pdf}} & 128 & \phantom{0}75.5 (59.0 \%) $\pm$ 7.4 & 90.8 $\pm$ 19.3 & 49.5 $\pm$ 19.8 & 451.9 $\pm$ 71.7 & 791.5 $\pm$ 202.1 & 542.7 $\pm$ 88.6 & 841.0 $\pm$ 217.7 & 0.6 \\
             &  &  & 256 & 174.7 (68.3 \%) $\pm$ 7.5 & 193.6 $\pm$ 27.1 & 68.1 $\pm$ 16.9 & 218.6 $\pm$ 35.6 & 417.6 $\pm$ \phantom{0}97.4 & 412.2 $\pm$ 58.7 & 485.7 $\pm$ 113.7 & 0.8 \\
             &  &  & 512 & 378.4 (73.9 \%) $\pm$ 6.8 & 411.1 $\pm$ 64.7 & 124.1 $\pm$ 23.9 & 125.9 $\pm$ 24.8 & 224.3 $\pm$ \phantom{0}32.1 & 537.0 $\pm$ 71.1 & 348.4 $\pm$ \phantom{0}54.1 & \textbf{1.5} \\
             
            \cmidrule(lr){2-12}            
            
             & \multirow{3}{*}{\shortstack{W6:\\ den520d\\ $257 \times 256$\ (28,178)}} & \multirow{3}{*}{\includegraphics[scale=0.029]{figures/workspaces/den520d.pdf}} & 128 & \phantom{0}71.9 (56.2 \%) $\pm$ 3.2 & 341.8 $\pm$ \phantom{0}50.5 & 147.6 $\pm$ 17.7 & 556.4 $\pm$ 71.4 & 773.1 $\pm$ 106.9 & 898.2 $\pm$ 121.2 & 920.7 $\pm$ 117.7 & \textbf{1.0} \\
             &  &  & 256 & 153.8 (60.1 \%) $\pm$ 3.7 & 685.4 $\pm$ 121.5 & 253.9 $\pm$ 63.1 & 324.2 $\pm$ 56.5 & 506.4 $\pm$ \phantom{0}57.4 & 1009.6 $\pm$ 175.2 & 760.3 $\pm$ 103.4 & \textbf{1.3} \\
             &  &  & 512 & 331.1 (64.7 \%) $\pm$ 6.6 & 1292.4 $\pm$ 304.0 & 316.0 $\pm$ 29.2 & 188.7 $\pm$ \phantom{0}9.2 & 398.6 $\pm$ \phantom{0}37.9 & 1481.1 $\pm$ 301.9 & 714.6 $\pm$ \phantom{0}51.7 & \textbf{2.1} \\
             
            \cmidrule(lr){2-12}            
            
             & \multirow{3}{*}{\shortstack{W7:\\ warehouse-20-40-10-2-2\\ $164 \times 340$\ (38,756)}} & \multirow{3}{*}{\includegraphics[scale=0.029]{figures/workspaces/warehouse-20-40-10-2-2.pdf}} & 128 & \phantom{0}81.2 (63.5 \%) $\pm$ 2.6 & 467.1 $\pm$ \phantom{0}40.0 & 211.3 $\pm$ 22.2 & 595.6 $\pm$ 43.1 & 707.6 $\pm$ \phantom{0}82.5 & 1062.7 $\pm$ \phantom{0}75.2 & 918.9 $\pm$ \phantom{0}98.3 & \textbf{1.2} \\
             &  &  & 256 & 149.3 (58.3 \%) $\pm$ 8.1 & 1009.5 $\pm$ 165.2 & 336.0 $\pm$ 31.2 & 361.6 $\pm$ 25.2 & 502.5 $\pm$ \phantom{0}78.2 & 1371.1 $\pm$ 185.3 & 838.5 $\pm$ 100.4 & \textbf{1.6} \\
             &  &  & 512 & 335.5 (65.5 \%) $\pm$ 7.2 & 1646.2 $\pm$ 343.7 & 394.7 $\pm$ 36.6 & 208.9 $\pm$ 18.0 & 375.2 $\pm$ \phantom{0}29.5 & 1855.1 $\pm$ 357.5 & 769.9 $\pm$ \phantom{0}54.7 & \textbf{2.4} \\
             
            \cmidrule(lr){2-12}            
            
             & \multirow{3}{*}{\shortstack{W8:\\ brc202d\\ $481 \times 530$\ (43,151)}} & \multirow{3}{*}{\includegraphics[scale=0.029]{figures/workspaces/brc202d.pdf}} & 128 & \phantom{0}65.1 (50.9 \%) $\pm$ 3.5 & 1019.0 $\pm$ 160.9 & 395.2 $\pm$ \phantom{0}70.7 & 915.9 $\pm$ 116.4 & 1444.9 $\pm$ 232.0 & 1934.9 $\pm$ 274.3 & 1840.1 $\pm$ 273.6 & \textbf{1.1} \\
             &  &  & 256 & 142.0 (55.5 \%) $\pm$ 8.7 & 1880.7 $\pm$ 275.0 & 697.7 $\pm$ 152.3 & 512.5 $\pm$ \phantom{0}60.8 & 972.5 $\pm$ 147.1 & 2393.2 $\pm$ 326.5 & 1670.2 $\pm$ 230.3 & \textbf{1.4} \\
             &  &  & 512 & 309.6 (60.5 \%) $\pm$ 9.5 & 3026.8 $\pm$ 326.8 & 705.6 $\pm$ 112.5 & 302.6 $\pm$ \phantom{0}21.6 & 780.9 $\pm$ \phantom{0}98.5 & 3329.4 $\pm$ 322.3 & 1486.5 $\pm$ 163.1 & \textbf{2.2} \\
            \bottomrule
        \end{tabular}
   }
\vspace{-0.3cm}
\end{table*}

%% file: 4b_quadcopter_2d.tex
A Quadcopter's state $s^i_j = (x, y) \in W$ is its location in the workspace. 
Its set of motion primitives $M$ = \{$\mathtt{Halt (H)}$, $\mathtt{MoveEast (ME)}$, $\mathtt{MoveNorth (MN)}$, $\mathtt{MoveWest (MW)}$, $\mathtt{MoveSouth (MS)}$\}, where $\mathtt{ME, MN, MW}$, and $\mathtt{MS}$ move it to the immediate \textit{east, north, west}, and \textit{south} cell, respectively. 

%% file: 7b_evaluation_ext.tex
\subsubsection{Effects of the non-participants on the participants}
\label{subsubsec:nonpar_par}


\begin{figure*}[t]
	\centering
	\includegraphics[scale=0.38]{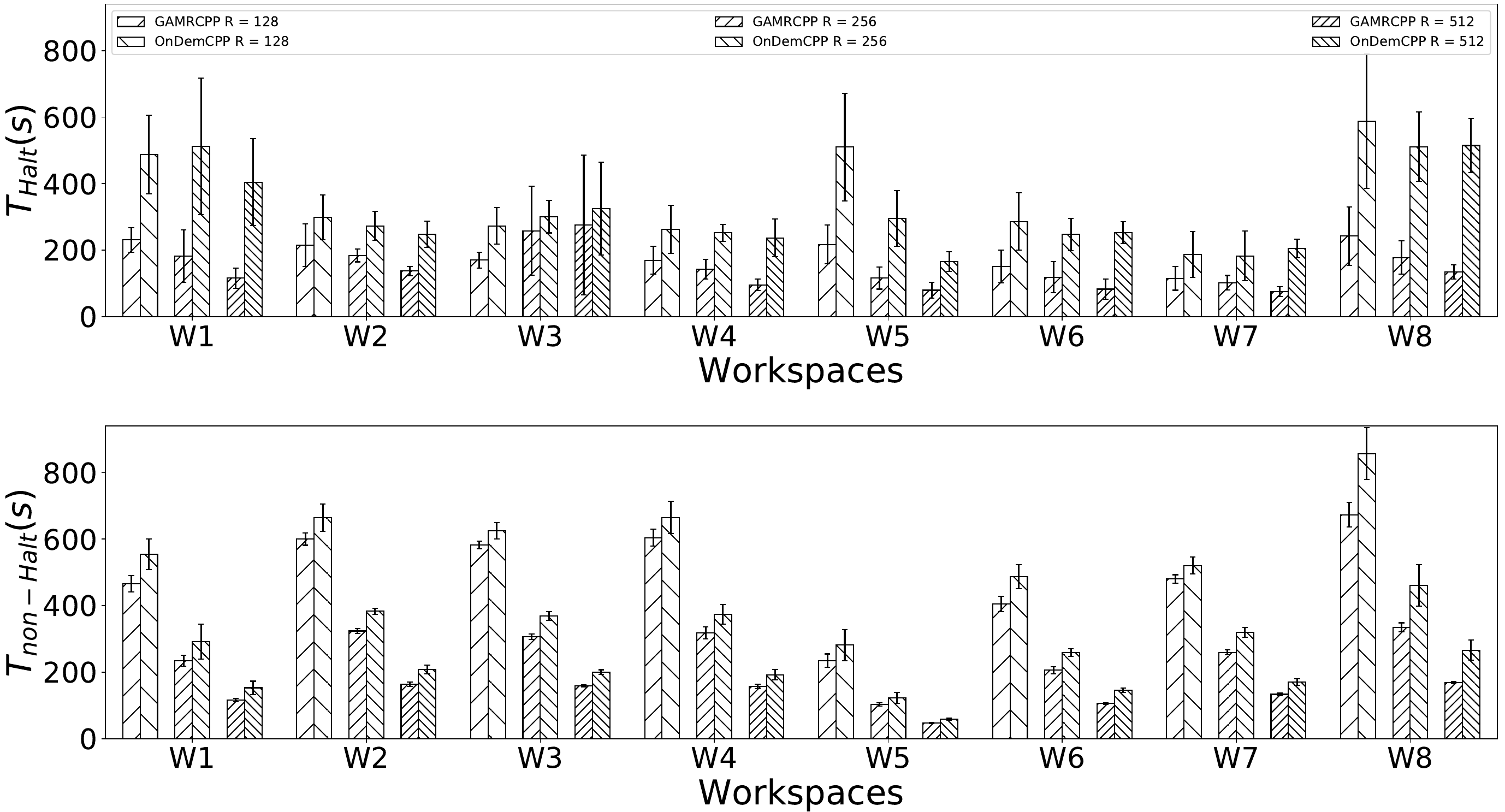}		
	\caption{Decomposition of $T_p$ into $T_{Halt}$ and $T_{non-Halt}$}
	\label{fig:halt_nonhalt}
\end{figure*}

Recall that the CP keeps the goals $\widetilde{W_g}$ reserved for the non-participants $\overline{I_{par}}$ in a horizon. 
So, the CP cannot assign these reserved goals to the participants $I_{par}$ in \FnCOPForPar (Algorithm \ref{algo:cop_for_par}) even if some participants can visit some reserved goals in a more cost-optimal manner than their corresponding non-participants. 
Consequently, the CP assigns the participants to the unreserved goals $W_g \setminus \widetilde{W_g}$, which could be far away. 
Thus, the participants' optimal paths $\Phi$, consisting of non-$\mathtt{Halt}$ moves, become marginally longer in \FnOnDemCPP as $T_{non-Halt}$ justifies in Figure~\ref{fig:halt_nonhalt}. 



\begin{figure*}[t]
	\centering
	\includegraphics[scale=0.38]{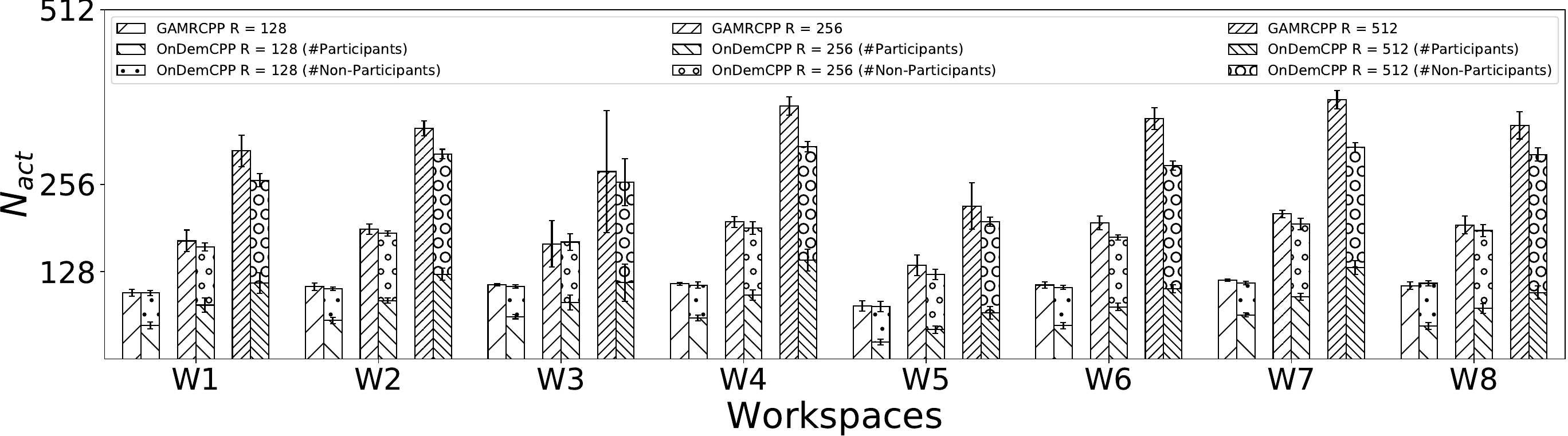}
	\caption{The number of active robots ($N_{act}$) per horizon}
	\label{fig:active_count}
\end{figure*}

Also, recall that the CP does not change the non-participants' remaining paths $\Sigma_{rem}$ in a horizon. 
So, it makes the participants' optimal paths $\Phi$ inter-robot collision-free in \FnCFPForPar (Algorithm \ref{algo:cfp_for_par}) by either prefixing the optimal paths with offsets (i.e., $\mathtt{Halt}$ moves) or inactivating some active participants in $\Phi$. 
Firstly, $T_{Halt}$ in Figure~\ref{fig:halt_nonhalt} justifies that the CP prefixes the participants' optimal paths with a substantial number of $\mathtt{Halt}$ moves to generate their collision-free paths $\Sigma'$. 

Lastly, in Figure~\ref{fig:active_count}, we show the number of active robots ($N_{act}$) per horizon to measure the extent of inactivation of the active participants in $\Phi$ during prioritized planning in a horizon. 
Unlike \FnGAMRCPP, where all $R$ robots participate in a horizon, only $R^*$ robots participate in \FnOnDemCPP. 
So, during prioritized planning, the CP may inactivate some active participants in $\Phi$ while finding their collision-free paths $\Sigma'$ under the added constraint in the form of the remaining paths of the $R - R^*$ non-participants. 
Figure~\ref{fig:active_count} justifies that the inactivation becomes more severe in \FnOnDemCPP as $R$ increases because the constraint also gets stronger for a larger number of non-participants. 


%% file: 8_conclusion.tex
\section{Conclusion}
\label{sec:conclusion}

We have proposed a centralized online on-demand CPP algorithm that uses a goal assignment-based prioritized planning method at its core. 
Our CP guarantees complete coverage of an unknown workspace with multiple homogeneous failure-free robots. 
Experimental results demonstrate its superiority over its counterpart by decreasing the mission time significantly in large workspaces with hundreds of robots. 
In the future, we will extend our work to an approach dealing with simultaneous path planning and plan execution.